\documentclass{article}

\usepackage{microtype}
\usepackage{graphicx}
\usepackage{subcaption}
\usepackage{amsmath}
\usepackage{amssymb}
\usepackage{amsthm}
\usepackage{booktabs} 

\usepackage{hyperref}
\usepackage{color}

\usepackage[accepted]{icml2019}

\def \xb {\mathbf{x}}

\def \cX {\mathcal{X}}

\def \btheta {{\bm \theta}}
\def \EE {\mathbb{E}}
\def \gb {\mathbf{g}}
\def \hgb {\hat \gb}
\def \bbf {{\bar f}}

\def \hxb {\hat \xb}
\def \cB {\mathcal{B}}
\newcommand{\la}{\langle}
\newcommand{\ra}{\rangle}
\usepackage{graphicx}
\usepackage{subcaption}
\usepackage{threeparttable}

\usepackage{amsmath}
\usepackage{amssymb}
\usepackage{amsfonts}
\usepackage{amsthm}
\usepackage{multirow}
\usepackage{bm}
\usepackage{dsfont}
\usepackage{diagbox}

\DeclareMathOperator*{\argmax}{argmax}

\newtheorem{theorem}{Theorem}

\newtheorem{lemma}{Lemma}

\newtheorem{assumption}{Assumption}

\icmltitlerunning{On the Convergence and Robustness of Adversarial Training}

\begin{document}

\twocolumn[
\icmltitle{On the Convergence and Robustness of Adversarial Training}



\icmlsetsymbol{equal}{*}

\begin{icmlauthorlist}
\icmlauthor{Yisen Wang}{equal,1}
\icmlauthor{Xingjun Ma}{equal,2}
\icmlauthor{James Bailey}{2}
\icmlauthor{Jinfeng Yi}{1}
\icmlauthor{Bowen Zhou}{1}
\icmlauthor{Quanquan Gu}{3}
\end{icmlauthorlist}

\icmlaffiliation{1}{JD.com}
\icmlaffiliation{2}{The University of Melbourne}
\icmlaffiliation{3}{The University of California, Los Angeles}

\icmlcorrespondingauthor{Quanquan Gu}{qgu@cs.ucla.edu}

\icmlkeywords{Adversarial Examples, Adversarial Defense, Adversarial Training, Optimization, Deep Neural Networks}

\vskip 0.3in
]



\printAffiliationsAndNotice{\icmlEqualContribution} 

\begin{abstract}
Improving the robustness of deep neural networks (DNNs) to adversarial examples is an important yet challenging problem for secure deep learning. Across existing defense techniques, adversarial training with Projected Gradient Decent (PGD) is amongst the most effective. Adversarial training solves a min-max optimization problem, with the \textit{inner maximization} generating adversarial examples by maximizing the classification loss, and the \textit{outer minimization} finding model parameters by minimizing the loss on adversarial examples generated from the inner maximization. A criterion that measures how well the inner maximization is solved is therefore crucial for adversarial training. In this paper, we propose such a criterion, namely First-Order Stationary Condition for constrained optimization (FOSC), to quantitatively evaluate the convergence quality of adversarial examples found in the inner maximization. With FOSC, we find that to ensure better robustness, it is essential to use adversarial examples with better convergence quality at the \textit{later stages} of training. Yet at the early stages, high convergence quality adversarial examples are not necessary and may even lead to poor robustness. Based on these observations, we propose a \textit{dynamic} training strategy to gradually increase the convergence quality of the generated adversarial examples, which significantly improves the robustness of adversarial training. Our theoretical and empirical results show the effectiveness of the proposed method. 
\end{abstract}

\section{Introduction}\label{sec:intro}
Although deep neural networks (DNNs) have achieved great success in a number of fields such as computer vision \cite{he2016deep} and natural language processing \cite{devlin2018bert}, they are vulnerable to adversarial examples crafted by adding small, human imperceptible adversarial perturbations to normal examples \cite{szegedy2013intriguing, goodfellow2014explaining}. Such vulnerability of DNNs raises security concerns about their practicability in security-sensitive applications such as face recognition \cite{kurakin2016adversarial} and autonomous driving \cite{chen2015deepdriving}. Defense techniques that can improve DNN robustness to adversarial examples have thus become crucial for secure deep learning.

There exist several defense techniques (\textit{i.e.}, ``defense models"), such as input denoising \cite{guo2018countering}, gradient regularization \cite{papernot2017practical}, and adversarial training \cite{madry2017towards}. 
However, many of these defense models  provide either only marginal robustness or have been evaded by new attacks \cite{athalye2018obfuscated}. One defense model that demonstrates moderate robustness, and has thus far not been comprehensively attacked, is \textit{adversarial training} \cite{athalye2018obfuscated}. Given a $C$-class dataset $S = \{(\xb_i^0, y_i)\}_{i=1}^n$ with $\xb_i^0 \in \mathbb{R}^d$ as a normal example in the $d$-dimensional input space and $y_i \in \{1, \cdots, C\}$ as its associated label, the objective of adversarial training is to solve the following \textit{min-max optimization} problem:
\begin{eqnarray}\label{eq:minimax}
    \min_\btheta \frac{1}{n} \sum_{i=1}^n \max_{\|\xb_i - \xb_i^0\|_\infty \leq \epsilon} \ell(h_\btheta (\xb_i), y_i),
\end{eqnarray}
where $h_\btheta: \mathbb{R}^d \to \mathbb{R}^C$ is the DNN function, $\xb_i$ is the adversarial example of $\xb_i^0$, $\ell(h_\btheta(\xb_i),y_i)$ is the loss function on the adversarial example $(\xb_i,y_i)$, and $\epsilon$ is the maximum perturbation constraint\footnote{We only focus on the infinity norm constraint in this paper, but our algorithms and theory apply to other norms as well.}.
The \textit{inner maximization} problem is to find an adversarial example $\xb_i$ within the $\epsilon$-ball around a given normal example $\xb_i^0$ (\textit{i.e.}, $\|\xb_i - \xb_i^0\|_\infty \leq \epsilon$) that maximizes the classification loss $\ell$. It is typically nonconcave with respect to the adversarial example. 
On the other hand, the \textit{outer minimization} problem is to find model parameters that minimize the loss $\ell$ on adversarial examples $\{\xb_i\}_{i=1}^n$ that generated from the inner maximization. This is the problem of training a robust classifier on adversarial examples. Therefore, how well the inner maximization problem is solved directly affects the performance of the outer minimization, \textit{i.e.}, the robustness of the classifier.

Several attack methods have been used to solve the inner maximization problem, such as Fast Gradient Sign Method (FGSM) \cite{goodfellow2014explaining} and Projected Gradient Descent (PGD) \cite{madry2017towards}. However, the degree to which they solve the inner maximization problem has not been thoroughly studied. Without an  appropriate criterion to measure how well the inner maximization is solved, the adversarial training procedure is difficult to monitor or improve. In this paper, we propose such a criterion, namely First-Order Stationary Condition for constrained optimization (FOSC), to measure the convergence quality of the adversarial examples found in the inner maximization. Our proposed FOSC facilitates monitoring and understanding adversarial training through the lens of convergence quality of the inner maximization, and this in turn motivates us to propose an improved training strategy for better robustness.
Our main contributions are as follows:
\begin{itemize}
  \item We propose a principled criterion FOSC to measure the convergence quality of adversarial examples found in the inner maximization problem of adversarial training. It is well-correlated with the adversarial strength of adversarial examples, and is also a good indicator of the robustness of adversarial training.
  \item With FOSC, we find that better robustness of adversarial training is associated with training on adversarial examples with better convergence quality in the \textit{later stages}. However, in the early stages, high convergence quality adversarial examples are not necessary and can even be harmful.
  \item We propose a \textit{dynamic} training strategy to gradually increase the convergence quality of the generated adversarial examples and provide a theoretical guarantee on the overall (min-max) convergence. Experiments show that \textit{dynamic} strategy significantly improves the robustness of adversarial training.
\end{itemize}

\section{Related Work}

\subsection{Adversarial Attack}
Given a normal example $(\xb_i^0, y_i)$ and a DNN $h_\btheta$, the goal of an attacking method is to find an adversarial example $\xb_i$ that remains in the $\epsilon$-ball centered at $\xb^0$ ($\|\xb_i - \xb_i^0\|_\infty \leq \epsilon$) but can fool the DNN to make an incorrect prediction ($h_\btheta(\xb_i) \neq y_i$). A wide range of attacking methods have been proposed for the crafting of adversarial examples. Here, we only mention a selection.

\textbf{Fast Gradient Sign Method (FGSM).} FGSM perturbs normal examples $\xb^0$ for one step ($\xb^1$) by the amount $\epsilon$ along the gradient direction \cite{goodfellow2014explaining}:
\begin{equation}
    \xb^1 = \xb^0 + \epsilon \cdot \text{sign}(\nabla_{\xb} \ell(h_{\btheta}(\xb^0), y)).
\end{equation}

\textbf{Projected Gradient Descent (PGD).} PGD perturbs normal example $\xb^0$ for a number of steps $K$ with smaller step size. After each step of perturbation, PGD projects the adversarial example back onto the $\epsilon$-ball of $\xb^0$, if it goes beyond the $\epsilon$-ball \cite{madry2017towards}:
\begin{equation}
\xb^{k} = \Pi \big( \xb^{k-1} + \alpha \cdot \text{sign}(\nabla_{\xb} \ell(h_{\btheta}(\xb^{k-1}), y)) \big),
\end{equation}
where $\alpha$ is the step size, $\Pi(\cdot)$ is the projection function, and $\xb^{k}$ is the adversarial example at the $k$-th step.

There are also other types of attacking methods, \textit{e.g.}, Jacobian-based Saliency Map Attack (JSMA) \citep{papernot2016limitations}, C\&W attack \citep{carlini2017towards} and Frank-Wolfe based attack \citep{chen2018frank}. PGD is regarded as the strongest first-order attack, and C\&W is among the strongest attacks to date.

\subsection{Adversarial Defense}
A number of defense models have been developed such as defensive distillation \citep{papernot2016distillation}, feature analysis \citep{xu2017feature, ma2018characterizing}, input denoising \cite{guo2018countering,liao2018defense,samangouei2018defensegan}, gradient regularization \cite{gu2014towards, papernot2017practical, tramer2017ensemble, ross2018improving}, model compression \cite{liu2018security,dascompression,rakin2018defend} and adversarial training \cite{goodfellow2014explaining, nokland2015improving, madry2017towards}, among which adversarial training is the most effective. 

Adversarial training improves the model robustness by training on adversarial examples generated by FGSM and PGD \cite{goodfellow2014explaining, madry2017towards}. 
\citeauthor{tramer2017ensemble} (\citeyear{tramer2017ensemble}) proposed an ensemble adversarial training on adversarial examples generated from a number of pretrained models. \citeauthor{kolter2017provable} (\citeyear{kolter2017provable}) developed a provable robust model that minimizes worst-case loss over a convex outer region. In a recent study by \cite{athalye2018obfuscated}, adversarial training on PGD adversarial examples was demonstrated to be the state-of-of-art defense model. Several improvements of PGD adversarial training have also been proposed, such as Lipschitz regularization \cite{cisse2017parseval, hein2017formal, yan2018deep, farnia2018generalizable}, and curriculum adversarial training \cite{cai2018curriculum}.

Despite these studies, a deeper understanding of adversarial training and a clear direction for further improvements is largely missing. The inner maximization problem in Eq. \eqref{eq:minimax} lacks an effective criterion that can quantitatively measure the convergence quality of training adversarial examples generated by different attacking methods (which in turn influences the analysis of the whole min-max problem).  In this paper, we propose such a criterion and provide new understanding of the robustness of adversarial training. We design a dynamic training strategy that significantly improves the robustness of the standard PGD adversarial training.

\section{Evaluation of the Inner Maximization}

\subsection{Quantitative Criterion: FOSC}\label{sec:criterion_definition}
In Eq. \eqref{eq:minimax}, the inner maximization problem is a constrained optimization problem, and is in general globally nonconcave. Since the gradient norm of $h$ is not an appropriate criterion for nonconvex/nonconcave constrained optimization problems, inspired by Frank-Wolfe gap \cite{frank1956algorithm}, we propose a First-Order Stationary Condition for constrained optimization (FOSC) as the convergence criterion for the inner maximization problem, which is affine invariant and not tied to any specific choice of norm:
\begin{equation}\label{eq:def_criterion}
    c(\xb^k) = \max_{\xb \in \cX} \la \xb - \xb^k, \nabla_\xb f(\btheta,\xb^k)\ra,
\end{equation}
where $\cX = \{\xb| \|\xb-\xb^0\|_\infty \leq \epsilon \}$ is the input domain of the $\epsilon$-ball around normal example $\xb^0$, $f(\btheta,\xb^k) = \ell(h_\btheta(\xb^k), y)$ and $\la \cdot \ra$ is the inner product. Note that $c(\xb^k) \geq 0$, and a smaller value of $c(\xb^k)$ indicates a better solution of the inner maximization (or equivalently, better convergence quality of the adversarial example $\xb^k$).

The criterion FOSC in Eq. \eqref{eq:def_criterion} can be shown to have the following closed-form solution:
\begin{align*}
    c(\xb^k) &= \max_{\xb \in \cX} \la \xb - \xb^k, \nabla_\xb f(\btheta,\xb^k)\ra \nonumber \\
    & = \max_{\xb \in \cX} \la \xb - \xb^0+ \xb^0-\xb^k, \nabla_\xb f(\btheta,\xb^k)\ra \nonumber \\
    & =  \max_{\xb \in \cX} \la \xb - \xb^0, \nabla_\xb f(\btheta,\xb^k)\ra\nonumber\\
    &\qquad+ \la \xb^k-\xb^0, -\nabla_\xb f(\btheta,\xb^k)\ra \nonumber \\
    & = \epsilon \|\nabla_\xb f(\btheta,\xb^k)\|_1 - \la \xb^k-\xb^0, \nabla_\xb f(\btheta,\xb^k)\ra.
\end{align*}
As an example-wise criterion, $c(\xb^k)$ measures the convergence quality of adversarial example $\xb^k$ with respect to both the perturbation constraint and the loss function. Optimal convergence where $c(\xb^k)=0$ can be achieved when 1) $\nabla f(\btheta,\xb^k)=0$, \textit{i.e.}, $\xb^k$ is a stationary point in the interior of $\cX$; or 2) $\xb^k-\xb^0=\epsilon \cdot \text{sign}(\nabla f(\btheta,\xb^k))$, that is, local maximum point of $f(\btheta,\xb^k)$ is reached on the boundary of $\cX$. The proposed criterion FOSC allows the monitoring of convergence quality of the inner maximization problem, and  provides a new perspective of adversarial training.

\subsection{FOSC View of Adversarial Training}\label{sec:convergence_view}
In this subsection, we will use FOSC to investigate the robustness and learning process of adversarial training. First though, we investigate its correlation with the traditional measures of accuracy and loss.

\textbf{FOSC View of Adversarial Strength.}
We train an 8-layer Convolutional Neural Network (CNN) on CIFAR-10 using 10-step PGD (PGD-10) with step size $\epsilon/4$, maximum perturbation $\epsilon=8/255$, following the standard setting in \citet{madry2017towards}.
We then apply the same PGD-10 attack on CIFAR-10 test images to craft adversarial examples, and divide the crafted adversarial examples into 20 consecutive groups of different convergence levels of FOSC value ranging from 0.0 to 0.1. The test accuracy and average loss of adversarial examples in each group are in Figure \ref{fig:criterion_acc}. We observe FOSC has a linear correlation with both accuracy and loss: the lower the FOSC, the lower (resp. higher) the accuracy (resp. loss).

We further show the intermediate perturbation steps of PGD for 20 randomly selected adversarial examples in Figure \ref{fig:criterion_loss}. As perturbation step increases, FOSC decreases consistently towards 0, while loss increases and stabilizes at a much wider range of values. Compared to the loss, FOSC provides a \textit{comparable} and \textit{consistent} measurement of adversarial strength: the closer the FOSC to 0, the stronger the attack. 

In summary, the proposed FOSC is well correlated with the adversarial strength and also more consistent than the loss, making it a promising tool to monitor adversarial training.

\begin{figure}[!t]
\centering
\begin{subfigure}{.52\linewidth}
  \centering
  \includegraphics[width=\textwidth]{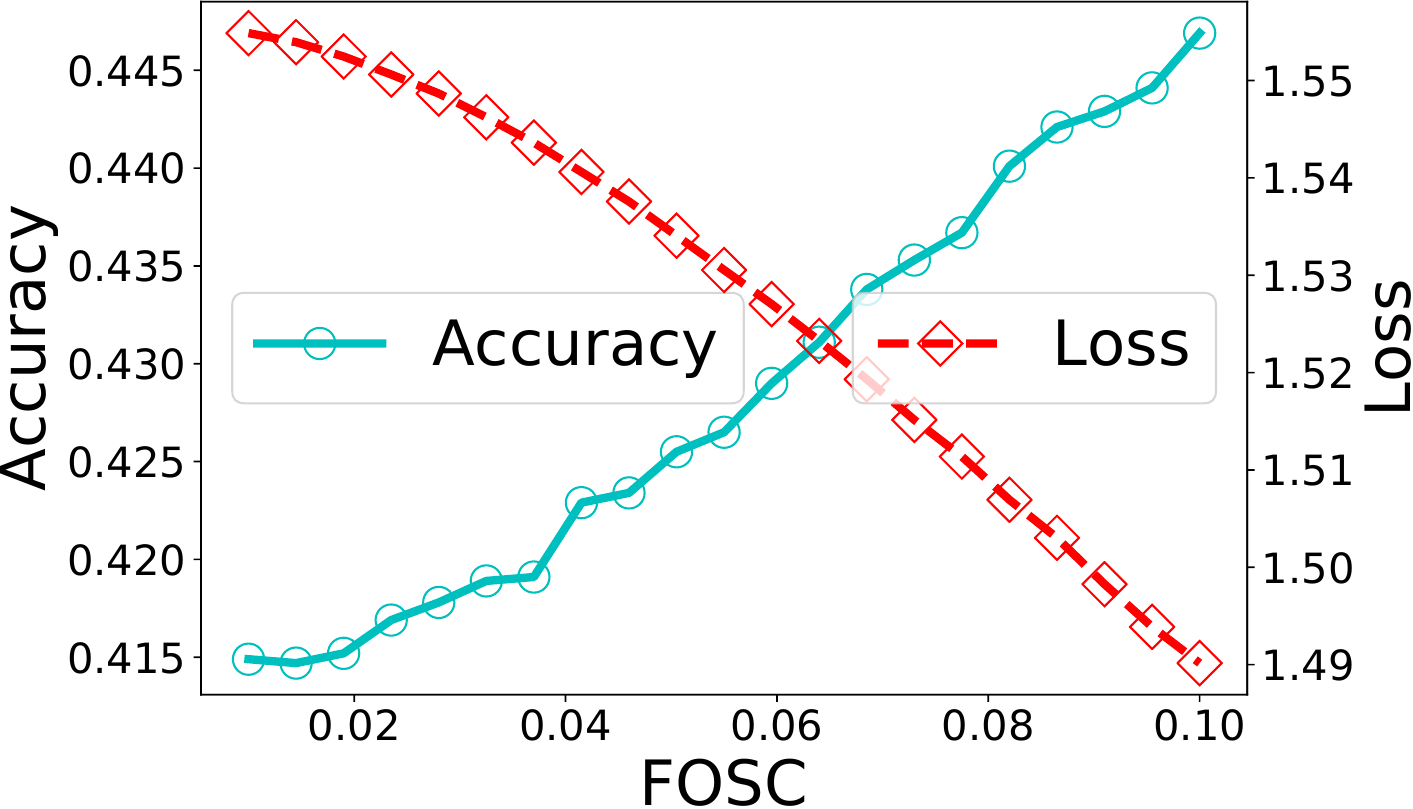}
  \caption{Accuracy, Loss vs. FOSC} 
  \label{fig:criterion_acc}
\end{subfigure}
\begin{subfigure}{.47\linewidth}
  \centering
  \includegraphics[width=\textwidth]{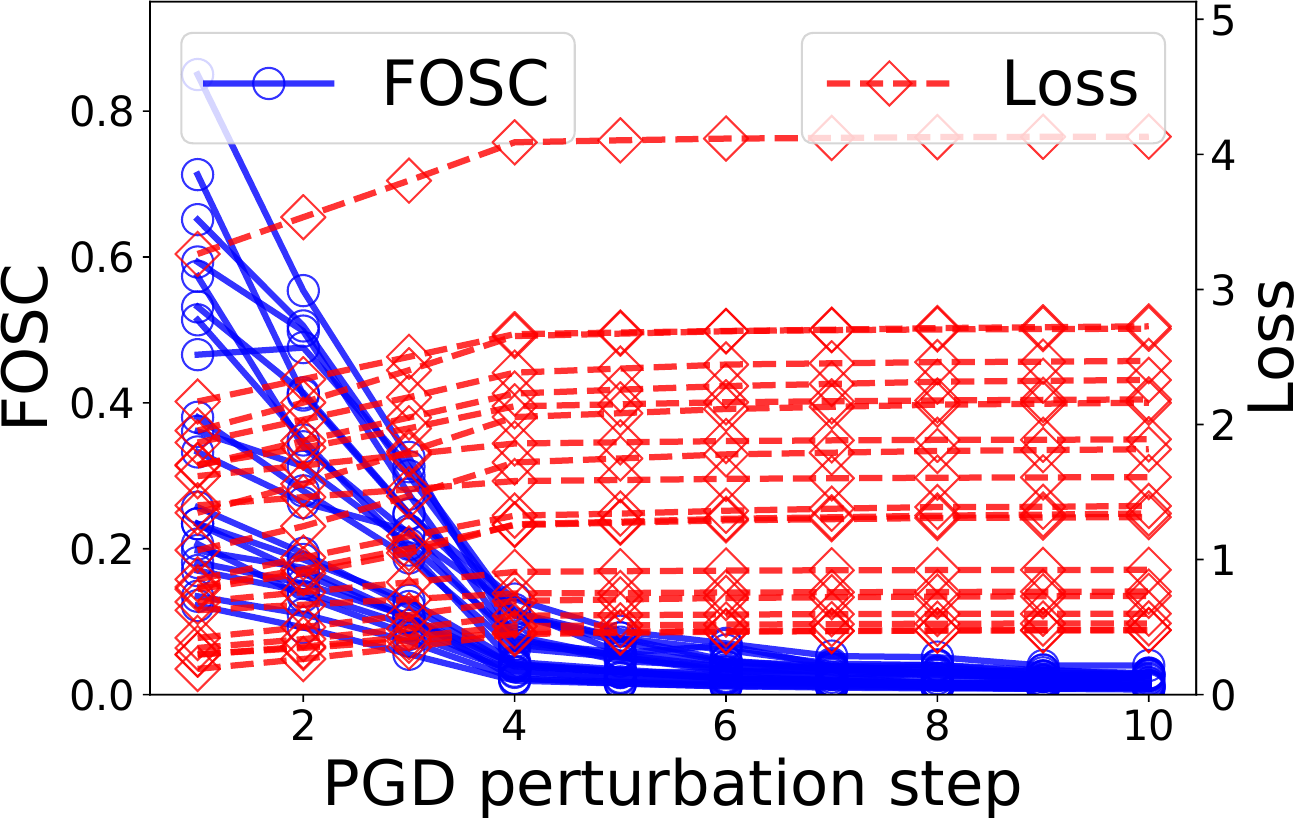}
  \caption{FOSC, Loss vs. Step} 
  \label{fig:criterion_loss}
\end{subfigure}
\caption{The correlation between convergence quality (FOSC) and adversarial strength (accuracy and loss). (a): For PGD-10 CIFAR-10 adversarial examples: the lower the FOSC (x-axis), the lower the accuracy (left y-axis) and the higher the loss (right y-axis). (b): For 20 randomly selected adversarial examples (each line is an example): PGD perturbation step (x-axis) versus FOSC (left y-axis) and loss (right y-axis). }
\label{fig:advs_strength}
\vspace{-0.15 in}
\end{figure}

\textbf{FOSC View of Adversarial Robustness.}
We first investigate the correlation among the final robustness of adversarial training, loss, and FOSC. In particular, we evaluate PGD adversarial training on CIFAR-10 in two settings: 1) varying PGD step size from $\epsilon, \epsilon/2$ to $\epsilon/8$ while fixing step number as 20, and 2) varying PGD step number from 10 to 40 while fixing step size as $\epsilon/6$. 
In each setting, we cross test (white-box) the robustness of the final model against PGD attacks in the same setting on CIFAR-10 test images. For each defense model, we also compute the distributions of FOSC and loss (using Gaussian kernel density estimation \cite{parzen1962estimation}) for the last epoch generated adversarial examples.

As shown in Figure~\ref{fig:motivation}, when varying step size, the best robustness against all test attacks is observed for PGD-$\epsilon/2$ or PGD-$\epsilon/4$ (Figure \ref{fig:pgd_size}), of which the FOSC distributions are more concentrated around 0 (Figure \ref{fig:criterion_size}) but their loss distributions are almost the same (Figure \ref{fig:loss_size}). 
When varying step number, the final robustness values are very similar (Figure \ref{fig:pgd_step}), which is also reflected by the similar FOSC distributions (Figure \ref{fig:criterion_step}), but the loss distributions are again almost the same (Figure \ref{fig:loss_step}). Revisiting Figure \ref{fig:pgd_step} where the step size is $\epsilon/6$, it is notable that increasing PGD steps only brings marginal or no robustness gain when the steps are more than sufficient to reach the surface of the $\epsilon$-ball: 12 steps of $\epsilon/6$ perturbation following the same gradient direction can reach the surface of the $\epsilon$-ball from any starting point. The above observations indicate that FOSC is a more reliable indicator of the final robustness of PGD adversarial training, compared to the loss.

\begin{figure}[!t]
\centering
\begin{subfigure}{.48\linewidth}
  \centering
  \includegraphics[width=\textwidth]{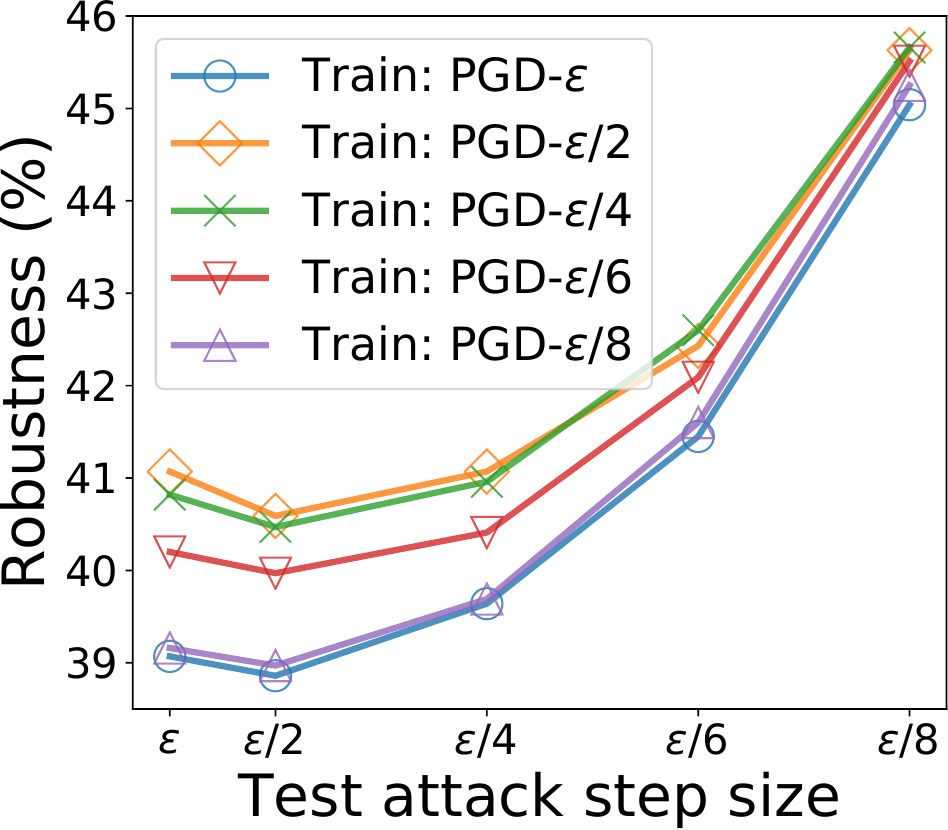}
  \vspace{-0.2 in}
  \caption{Robustness vs. Step size }
  \label{fig:pgd_size}
\end{subfigure}  
\begin{subfigure}{.48\linewidth}
  \centering  
  \includegraphics[width=\textwidth]{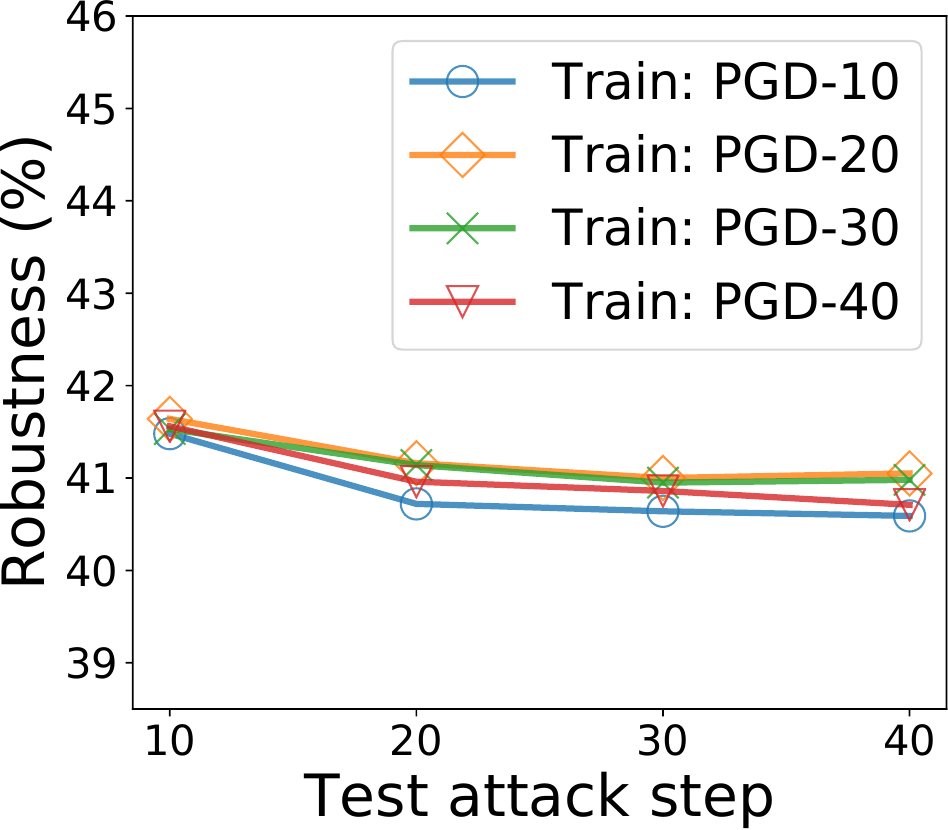}
  \vspace{-0.2 in}
  \caption{Robustness vs. Step number}
  \label{fig:pgd_step}
\end{subfigure} \\
\begin{subfigure}{.48\linewidth}
  \centering
  \includegraphics[width=\textwidth]{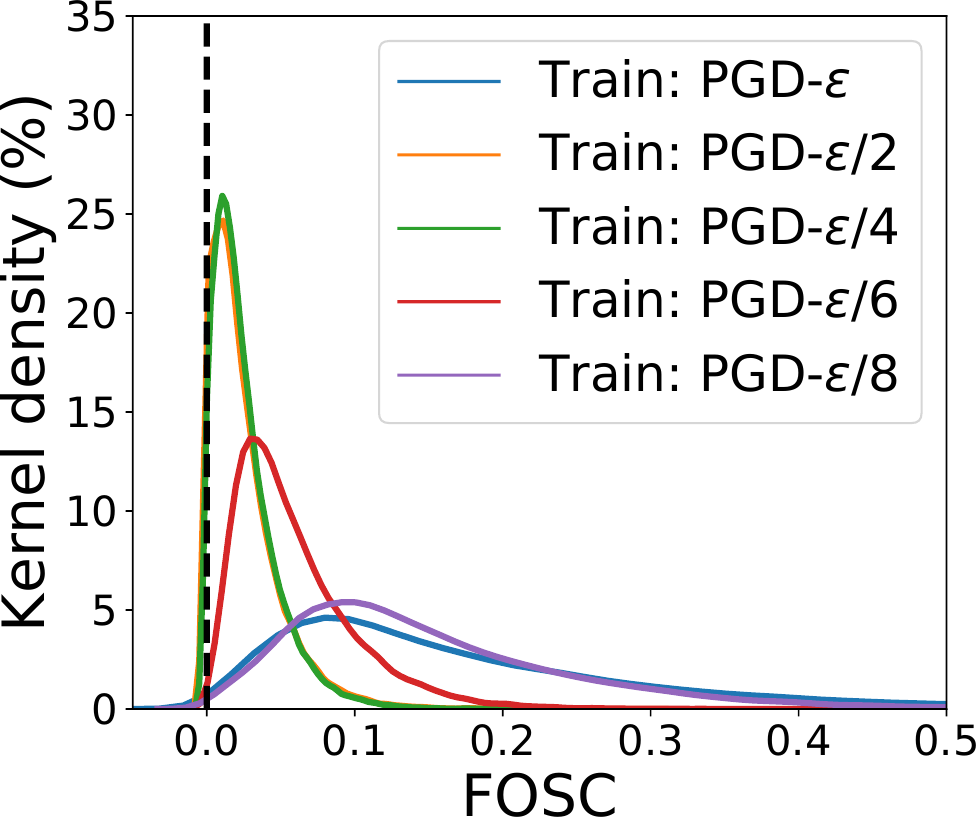}
  \vspace{-0.2 in}
  \caption{FOSC vs. Step size}
  \label{fig:criterion_size}
\end{subfigure}
\begin{subfigure}{.48\linewidth}
  \includegraphics[width=\textwidth]{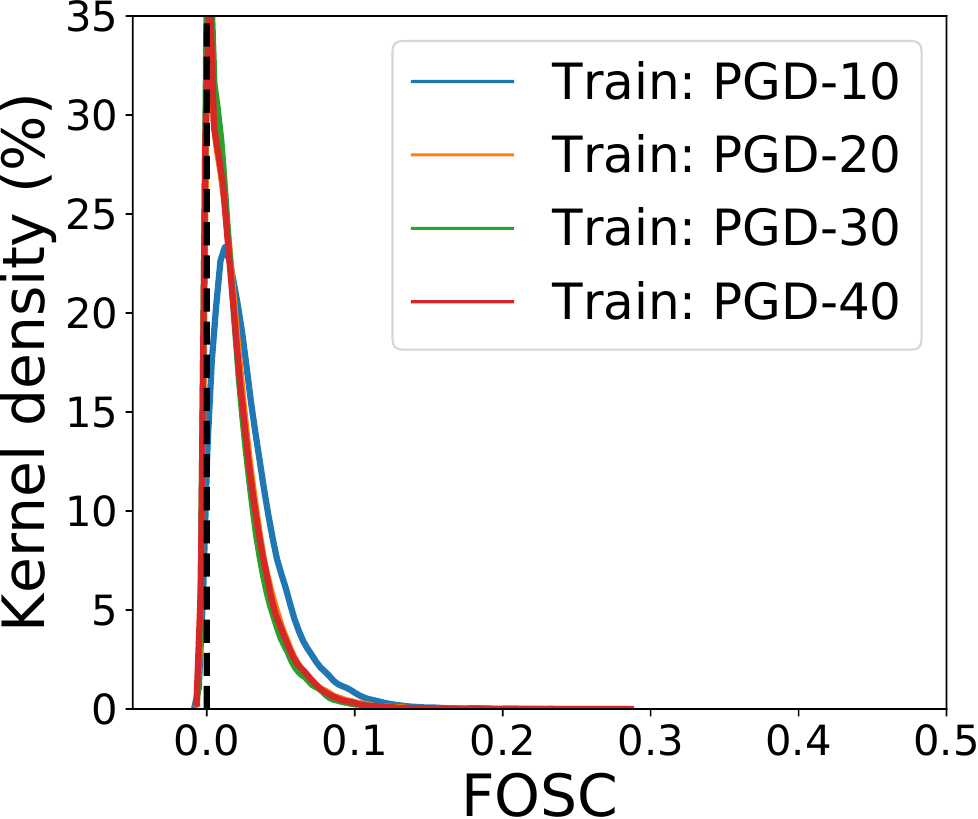}
  \vspace{-0.2 in}
  \caption{FOSC vs. Step number}
  \label{fig:criterion_step}
\end{subfigure}\\
\begin{subfigure}{.48\linewidth}
  \centering
  \includegraphics[width=\textwidth]{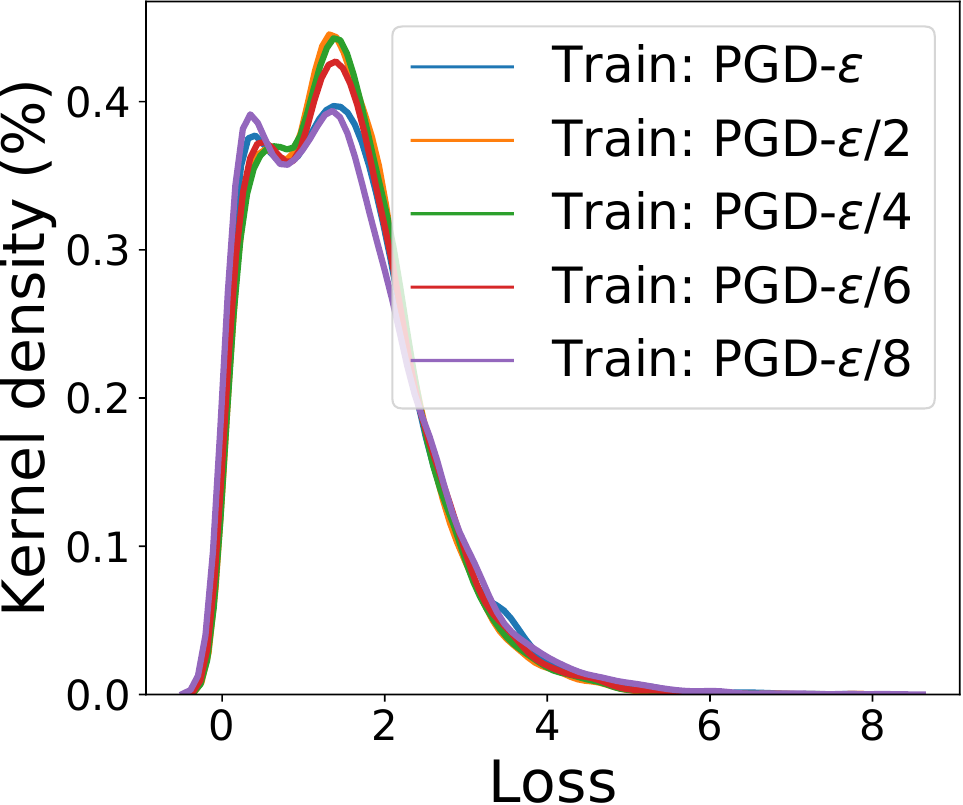}
  \vspace{-0.2 in}
  \caption{Loss vs. Step size}
  \label{fig:loss_size}
\end{subfigure}
\begin{subfigure}{.48\linewidth}
  \includegraphics[width=\textwidth]{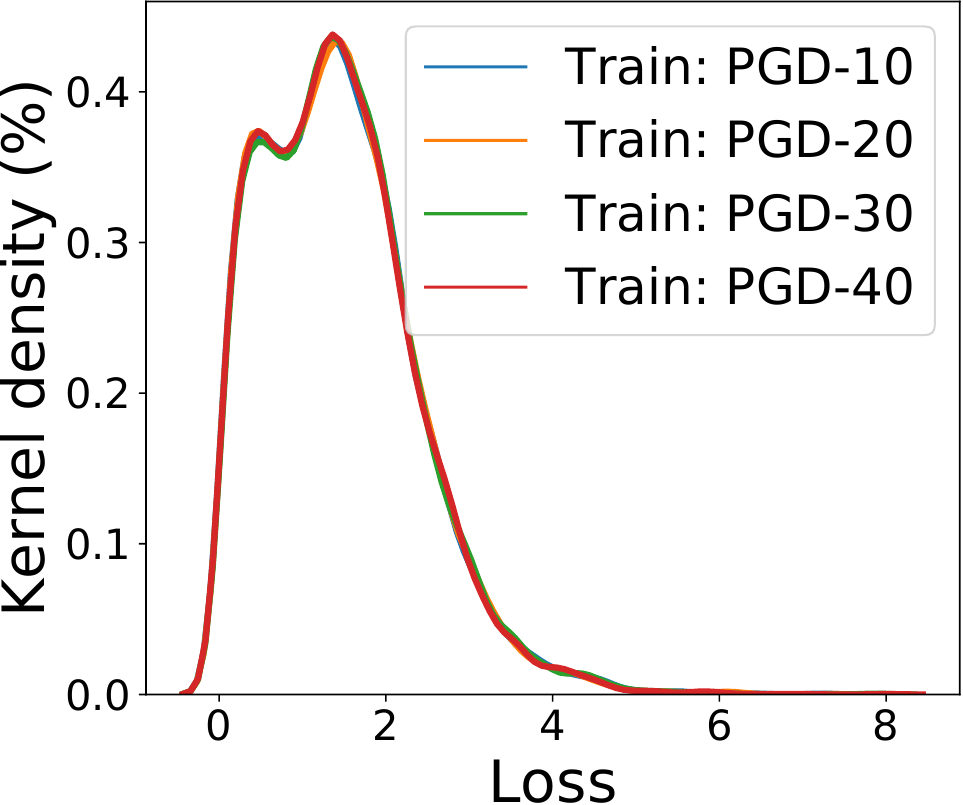}
  \vspace{-0.2 in}
  \caption{Loss vs. Step number}
  \label{fig:loss_step}
\end{subfigure}
\caption{
Robustness of PGD adversarial training with (a) varying step size (fixed step number 20), or (b) varying step number (fixed step size $\epsilon/6$). The FOSC distributions (c)/(d) reflect the robustness of adversarial training in (a)/(b), \textit{i.e.}, the lower the FOSC, the better the robustness. The loss distributions (e)/(f) are almost the same for different settings.}
\label{fig:motivation}
\vspace{-0.15 in}
\end{figure}

\textbf{Rethinking the Adversarial Training Process.}
To provide more insights into the learning process of adversarial training, we show the distributions of FOSC at three distinct learning stages in Figure \ref{fig:criterion_advs_train}: 1) early stage (epoch 10), 2) middle stage (epoch 60), and 3) later stage (epoch 100) (120 epochs in total). We only focus on two defense models: training with 10-step PGD-$\epsilon/4$ and 10-step PGD-$\epsilon/8$ (the best and worst model observed in Figure \ref{fig:pgd_size} respectively). In Figure \ref{fig:pgd_4_8}, for both models, FOSC at the early stage is significantly lower than the following two stages. Thus, at the early stage, both models can easily find high convergence quality adversarial examples for training; however, it becomes more difficult to do so at the following stages. This suggests overfitting to strong PGD adversarial examples at the early stage. To verify this, we replace the first 20 epochs of PGD-$\epsilon/4$ training with a much weaker FGSM (1 step perturbation of size $\epsilon$), denoted as ``FGSM-PGD", and show its robustness and FOSC distribution in Figure \ref{fig:fgsm_pgd_robustness} and \ref{fig:fgsm_pgd_criterion} respectively. We find that by simply using weaker FGSM adversarial examples at the early stage, the final robustness and the convergence quality of adversarial examples found by PGD at the later stage are both significantly improved. The FOSC density between $[0, 0.1]$ is improved to above $35\%$ (green solid line in Figure \ref{fig:fgsm_pgd_criterion}) from less than $30\%$ (green solid line in Figure \ref{fig:pgd_4_8}). This indicates strong PGD attacks are not necessary for the \textit{early stage} of training, or even deteriorate the robustness. In the next section, we will propose a dynamic training strategy to address this issue. 

\begin{figure}[!t]
\begin{subfigure}{.32\linewidth}
  \centering
  \includegraphics[width=\textwidth]{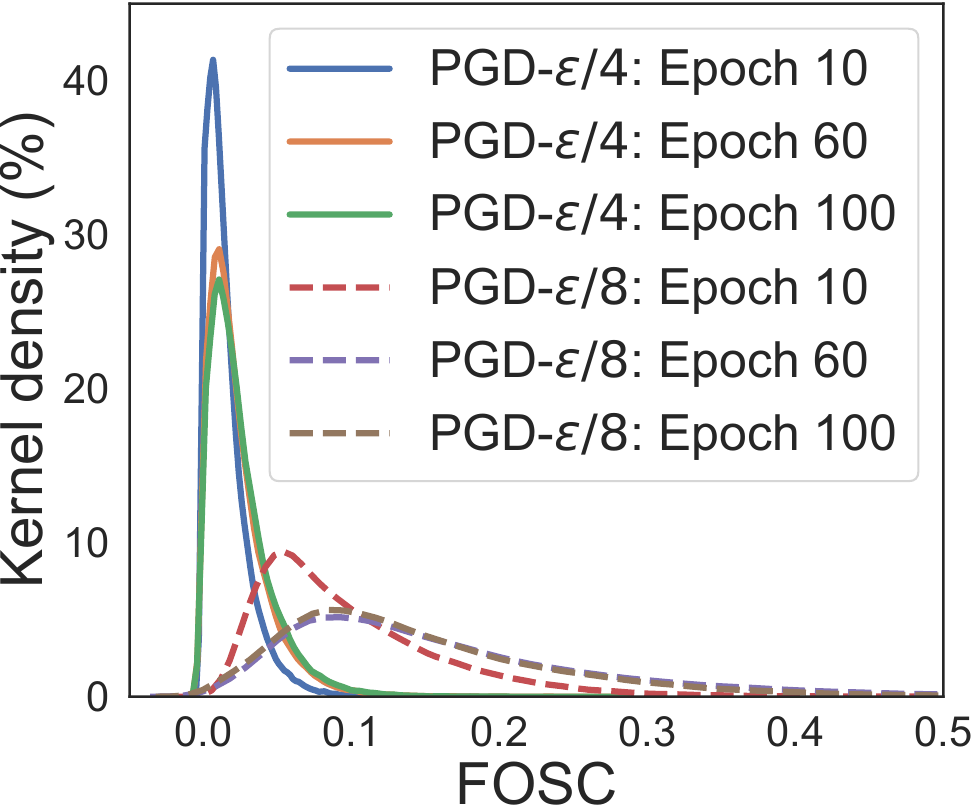}
  \vspace{-0.2 in}
  \caption{FOSC}
  \label{fig:pgd_4_8}
\end{subfigure}
\begin{subfigure}{.32\linewidth}
  \centering
  \includegraphics[width=\textwidth]{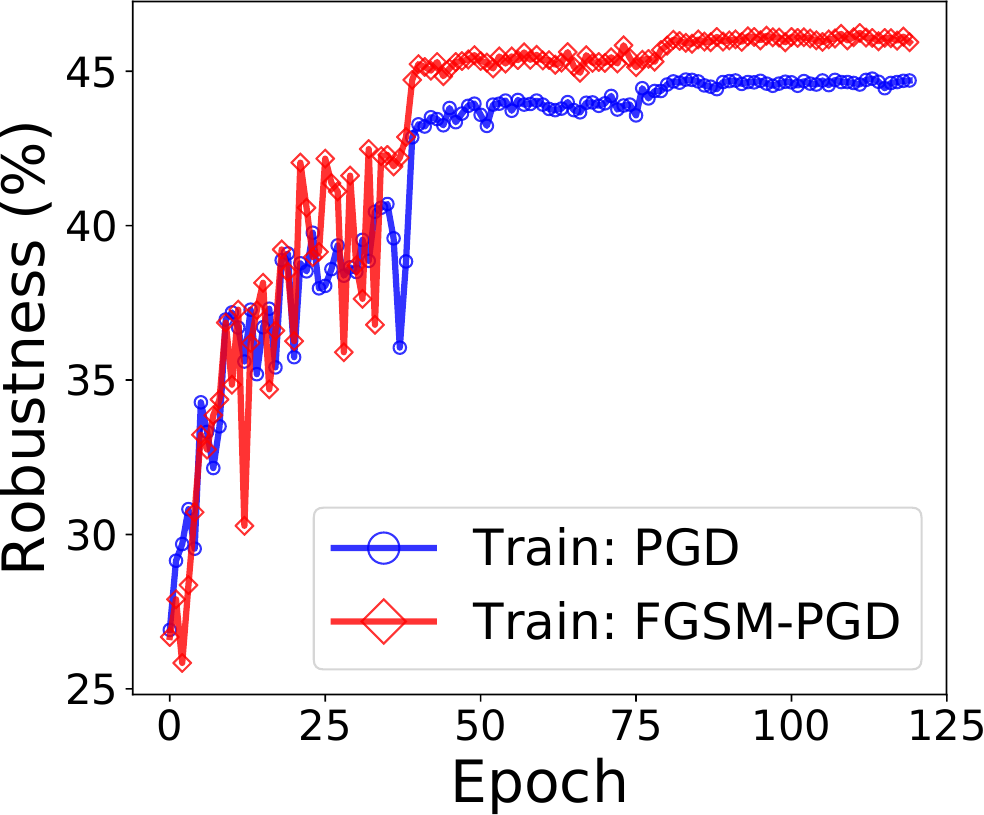}
  \vspace{-0.2 in}
  \caption{Robustness}
  \label{fig:fgsm_pgd_robustness}
\end{subfigure}
\begin{subfigure}{.32\linewidth}
  \includegraphics[width=\textwidth]{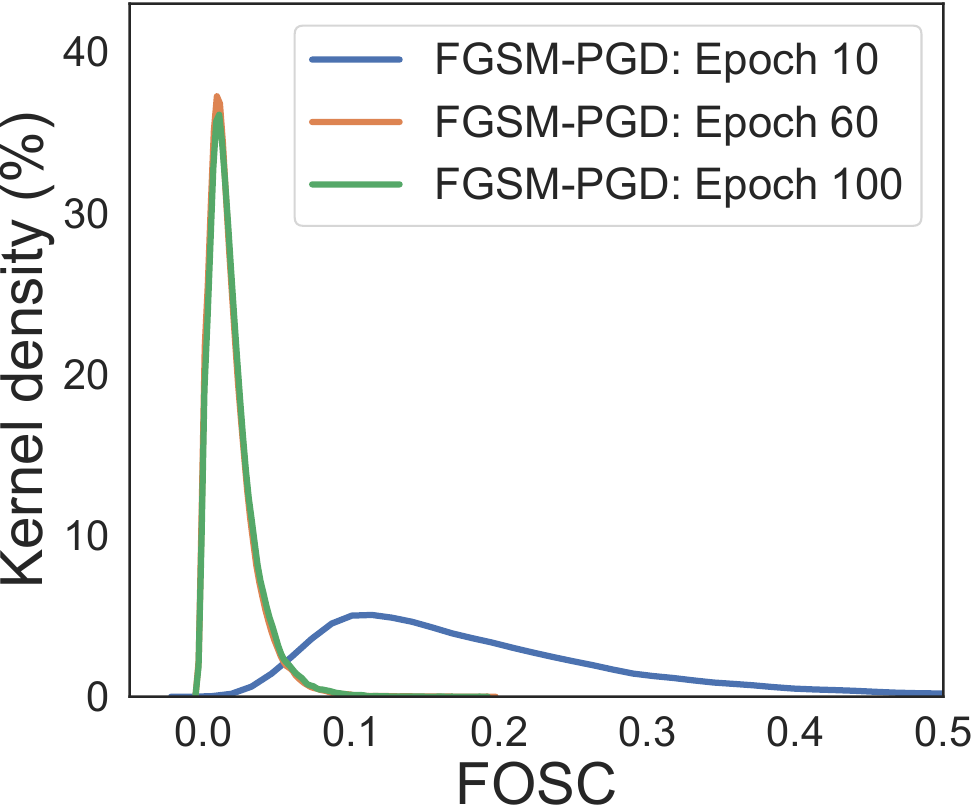}
  \vspace{-0.2 in}
  \caption{FOSC}
  \label{fig:fgsm_pgd_criterion}
\end{subfigure}
\caption{(a): FOSC distribution at intermediate epochs (10, 60, 100) for adversarial training of 10 steps PGD with step size $\epsilon/4$ (PGD-$\epsilon/4$) and step size $\epsilon/8$ (PGD-$\epsilon/8$); (b): The robustness of training with PGD and training with first FGSM then PGD; (c): FOSC distribution at intermediate epochs for training with first FGSM then PGD. Distributions at the $60$-th and $100$-th epochs overlap each other for (a)/(c).}
\label{fig:criterion_advs_train}
\vspace{-0.15in}
\end{figure}

\section{Dynamic Adversarial Training}\label{sec:dynamic}
In this section, we first introduce the proposed dynamic adversarial training strategy. Then we provide a theoretical convergence analysis of the min-max problem in Eq.~\eqref{eq:minimax} of the proposed approach.

\subsection{The Proposed Dynamic Training Strategy}
As mentioned in Section \ref{sec:convergence_view}, training on adversarial examples of better convergence quality at the \textit{later stages} leads to higher robustness. However, at the \textit{early stages}, training on high convergence quality adversarial examples may not be helpful. Recalling the criterion FOSC proposed in Section \ref{sec:criterion_definition}, we have seen that it is strongly correlated with adversarial strength. Thus, it can be used to monitor the strength of adversarial examples at a  fine-grained level. Therefore, we propose to train DNNs with adversarial examples of gradually decreasing FOSC value (increasing convergence quality), so as to ensure that the network is trained on weak adversarial examples at the early stages and strong adversarial examples at the later stages.  

Our proposed dynamic adversarial training algorithm is shown in Algorithm \ref{alg:dynamic}. The dynamic criterion FOSC $c_t = \max(c_{\max} - t \cdot c_{\max}/T', 0)$ controls the minimum FOSC value (maximum adversarial strength) of the adversarial examples at the $t$-th epoch of training ($T'$ is slightly smaller than total epochs $T$ to ensure the later stage can be trained on criterion 0). In the early stages of training, $c_t$ is close to the maximum FOSC value $c_{\max}$ corresponding to weak adversarial examples, it then decreases linearly towards zero as training progresses\footnote{This is only a simple strategy that works well in our experiments and other strategies could also work here.}, and is zero after the $T'$-th epoch of training. We use PGD to generate the training adversarial examples, however, at each perturbation step of PGD, we monitor the FOSC value and stop the perturbation process for adversarial example whose FOSC value is already smaller than $c_t$, enabled by an indicator control vector $V$. $c_{\max}$ can be estimated by the average FOSC value on a batch of weak adversarial examples such as FGSM.

\begin{algorithm}[!htbp]
   \caption{Dynamic Adversarial Training}
   \label{alg:dynamic}
\begin{algorithmic}
   \STATE {\bfseries Input:} Network $h_{\btheta}$, training data $S$, initial model parameters $\btheta^0$, step size $\eta_t$, mini-batch $\cB$, maximum FOSC value $c_{max}$, training epochs $T$, FOSC control epoch $T'$, PGD step $K$, PGD step size $\alpha$, maximum perturbation $\epsilon$.
  \FOR{$t = 0$ {\bfseries to} $T-1$}
    \STATE $c_t = \max(c_{\max} - t \cdot c_{\max}/T', 0)$
       \FOR{each batch $\xb^{0}_{\cB}$}
          \STATE $V = \mathds{1}_{\cB}$ ~~~~~\# \textit{control vector of all elements is 1}
           \WHILE{$\sum V > 0$ \& $k < K$}
           \STATE $\xb^{k+1}_{\cB} = \xb^{k}_{\cB} + V\cdot \alpha \cdot \text{sign}(\nabla_\xb \ell(h_{\btheta}(\xb^{k}_{\cB}), y))$
           \STATE $\xb^{k}_{\cB} = clip(\xb^{k}_{\cB}, \xb^{0}_{\cB} - \epsilon, \xb^{0}_{\cB} + \epsilon)$
           \STATE $V = \mathds{1}_{\cB}(c(\xb^{k}_{1 \cdots \cB}) \leq c_t)$ ~~~~~\# \textit{The element of $V$ becomes 0 at which FOSC is smaller than $c_t$}
           \ENDWHILE
        \STATE $\btheta^{t+1} = \btheta^t - \eta_t \gb(\btheta^t)$  ~~~\# \textit{$\gb(\btheta^t):$ stochastic gradient}
       \ENDFOR
   \ENDFOR
\end{algorithmic}
\end{algorithm}

\subsection{Convergence Analysis}
We provide a convergence analysis of our proposed dynamic adversarial training approach (as opposed to just the inner maximization problem) for solving the overall min-max optimization problem in Eq. \eqref{eq:minimax}.  Due to the nonlinearities in DNNs such as ReLU \cite{nair2010rectified} and max-pooling functions, the exact assumptions of Danskin's theorem \cite{danskin2012theory} do not hold. Nevertheless, given the criterion FOSC that ensures an approximate maximizer of the inner maximization problem, we can still provide a theoretical convergence guarantee. 

In detail, let $\xb_i^*(\btheta) = \argmax_{\xb_i \in \cX_i}f(\btheta,\xb_i)$ where $f(\btheta,\xb) = \ell(h_{\btheta}(\xb),y)$ is a shorthand notation for the classification loss function, $\cX_i = \{\xb | \|\xb - \xb_i^0\|_\infty \leq \epsilon \}$, and $\bbf_i(\btheta) = \max_{\xb_i \in \cX_i} f(\btheta,\xb_i) = f(\btheta,\xb_i^*(\btheta))$, then $\hxb_i(\btheta)$ is a $\delta$-approximate solution to
$\xb_i^*(\btheta)$, if it satisfies that
\begin{align}\label{eq:criterion}
    c(\hxb_i(\btheta)) = \max_{\xb \in \cX_i} \la \xb - \hxb_i(\btheta), \nabla_{\xb} f(\btheta,\hxb_i(\btheta))\ra \leq \delta.
\end{align}
In addition, denote the objective function in Eq. \eqref{eq:minimax} by $L_S(\btheta)$, and its gradient by $\nabla L_S(\btheta) = 1/n \sum_{i=1}^n \nabla \bbf_i(\btheta) = 1/n\sum_{i=1}^n \nabla_{\btheta} f(\btheta,\xb_i^*(\btheta))$. Let $\gb(\btheta) =  1/|\cB|\sum_{i \in \cB}\nabla \bbf_i(\btheta)$ be the stochastic gradient of $L_S(\btheta)$, where $\cB$ is the mini-batch. We have $\EE[\gb(\btheta)] = \nabla L_S(\btheta)$. Let $\nabla_{\btheta} f(\btheta,\hxb(\btheta))$ be the gradient of $f(\btheta,\hxb(\btheta))$ with respect to $\btheta$,
and $\hgb(\theta) = 1/|\cB| \sum_{i\in \cB} \nabla_{\btheta} f(\btheta,\hxb_i(\btheta))$ be the approximate stochastic gradient of $L_S(\btheta)$. 

Before we provide the convergence analysis, we first lay out a few assumptions that are needed for our analysis.

\begin{assumption}\label{assum:gradient Lipschitz}
The function $f(\btheta;\xb)$ satisfies the gradient Lipschitz conditions as follows
\begin{align*}
    \sup_{\xb}\|\nabla_{\btheta} f(\btheta,\xb)-\nabla_{\btheta} f(\btheta',\xb)\|_2 &\leq L_{\theta\theta} \|\btheta-\btheta'\|_2 \\
    \sup_{\btheta}\|\nabla_{\btheta} f(\btheta,\xb)-\nabla_{\btheta} f(\btheta,\xb')\|_2 &\leq L_{\theta x} \|\xb-\xb'\|_2 \\
    \sup_{\xb}\|\nabla_{\xb} f(\btheta,\xb)-\nabla_{\xb} f(\btheta',\xb)\|_2 &\leq L_{x\theta} \|\btheta-\btheta'\|_2,
\end{align*}
where $L_{\theta\theta}, L_{\theta x}, L_{x\theta}$ are positive constants.
\end{assumption}
Assumption \ref{assum:gradient Lipschitz} was made in \citet{sinha2018certifying}, which requires the loss function is smooth in the first and second arguments. While ReLU \cite{nair2010rectified} is non-differentiable, recent studies \citep{allen2018convergence,du2018gradient,zou2018stochastic,cao2019generalization} showed that the loss function of overparamterized deep neural networks is semi-smooth. This helps justify Assumption \ref{assum:gradient Lipschitz}.

\begin{assumption}\label{assum:strongly concave}
$f(\btheta,\xb)$ is locally $\mu$-strongly concave in $\cX_i = \{\xb: \|\xb-\xb_i^0\|_\infty \leq \epsilon\}$ for all $i \in [n]$, i.e., for any $\xb_1,\xb_2 \in \cX_i$, it holds that
\begin{align*}
\scalebox{0.9}{
    $f(\btheta,\xb_1) \leq f(\btheta,\xb_2) + \la\nabla_\xb f(\btheta,\xb_2),\xb_1-\xb_2\ra - \frac{\mu}{2}\|\xb_1-\xb_2\|_2^2$.}
\end{align*}
\end{assumption}
Assumption \ref{assum:strongly concave} can be verified using the relation between robust optimization and distributional robust optimization (refer to \citet{sinha2018certifying,lee2018minimax}). 

\begin{assumption}\label{assum:variance}
The variance of the stochastic gradient $\gb(\btheta)$ is bounded by a constant $\sigma^2>0$, 
\begin{align*}
    \EE[\| \gb(\btheta) - \nabla L_S(
\btheta)\|_2^2] \leq \sigma^2,
\end{align*}
where $\nabla L_S(
\btheta)$ is the full gradient.
\end{assumption}
Assumption \ref{assum:variance} is a common assumption made for the analysis of stochastic gradient based optimization algorithms.

\begin{theorem}\label{thm:main}
Suppose Assumptions \ref{assum:gradient Lipschitz}, \ref{assum:strongly concave} and \ref{assum:variance} hold. Let $\Delta = L_S(\btheta^0) - \min_{\btheta} L_S(\btheta)$. If the step size of the outer minimization is set to $\eta_t = \eta = \sqrt{\Delta/(TL\sigma^2)}$ and $T\geq \Delta L/\sigma^2$. Then the output of Algorithm \ref{alg:dynamic} satisfies
\begin{align*}
    \frac{1}{T} \sum_{t=0}^{T-1} \EE\big[ \|\nabla L_S(\btheta^t)\|_2^2\big] \leq 4\sigma\sqrt{\frac{L\Delta}{T}} + \frac{2 L_{\theta x}^2\delta}{\mu},
\end{align*}
where $L = (L_{\theta x}L_{ x\theta}/\mu+L_{\theta \theta})$.
\end{theorem}

The complete proof can be found in the supplementary material. Theorem \ref{thm:main} suggests that if the inner maximization is solved up to a precision so that the criterion FOSC is less than $\delta$,  Algorithm \ref{alg:dynamic} can converge to a first-order stationary point at a sublinear rate up to a precision of $ 2L_{\theta x}^2\delta/\mu$. In practice, if $\delta$ is sufficiently small such that $ 2L_{\theta x}^2\delta/\mu$ is small enough, Algorithm \ref{alg:dynamic} can find a robust model $\btheta^T$. This supports the validity of Algorithm \ref{alg:dynamic}.

\subsection{Relation to Curriculum Learning}
Curriculum learning \cite{bengio2009curriculum} is a learning paradigm in which a model learns from easy examples first then gradually learns from more and more difficult examples. For training with normal examples, it has been shown to be able to speed up convergence and improve generalization. This methodology has been adopted in many applications to enhance model training \cite{kumar2010self, jiang2015self}. The main challenge for curriculum learning is to define a proper criterion to determine the difficulty/hardness of training examples, so as to design a learning curriculum (\textit{i.e.}, a sequential ordering) mechanism. Our proposed criterion FOSC in Section \ref{sec:criterion_definition} can serve as such a difficulty measure for training examples, and our proposed dynamic approach can be regarded as one type of curriculum learning.

Curriculum learning was used in adversarial training in \citet{cai2018curriculum}, with the perturbation step of PGD as the difficulty measure. Their assumption is that more perturbation steps indicate stronger adversarial examples. 
However, this is not a reliable assumption from the FOSC view of the inner maximization problem: more steps may overshoot and result in suboptimal adversarial examples. Empirical comparisons with \cite{cai2018curriculum} will be shown in Sec. \ref{sec:experiments}. 

\section{Experiments}\label{sec:experiments}
In this section, we evaluate the robustness of our proposed training strategy (\textit{Dynamic}) compared with several state-of-the-art defense models, in both the white-box and black-box settings on benchmark datasets MNIST and CIFAR-10. We also provide analysis and insights on the robustness of different defense models. For all adversarial examples, we adopt the infinity norm ball as the maximum perturbation constraint \cite{madry2017towards}.

\textbf{Baselines.} 
The baseline defense models we use include  1) \textit{Unsecured}: unsecured training on normal examples; 2) \textit{Standard}: standard adversarial training with PGD attacks \cite{madry2017towards}; 3) \textit{Curriculum}: curriculum adversarial training with PGD attacks of gradually increasing the number of perturbation steps \cite{cai2018curriculum}.

\subsection{Robustness Evaluation}\label{sec:robustness}

\begin{table*}[!t]
    \vspace{-0.1 in}
    \caption{White-box robustness (accuracy (\%) on white-box test attacks) of different defense models on MNIST and CIFAR-10 datasets.}
    \vspace{0.1 in}
    \centering
    \label{tab:white_box_accuracy}
    \scalebox{0.99}[0.99]{
    \begin{tabular*}{\textwidth}{ @{\extracolsep{\fill}} lccccc|ccccc}
    \hline
    \multirow{2}[4]{*}{Defense} &\multicolumn{5}{c}{MNIST} &\multicolumn{5}{c}{CIFAR-10} \\
    \cmidrule(lr){2-6} \cmidrule(lr){7-11}
    & Clean & FGSM & PGD-10 & PGD-20 & C\&W$_{\infty}$ & Clean & FGSM & PGD-10 & PGD-20 & C\&W$_{\infty}$  \\
    \hline
    \textit{Unsecured} & \textbf{99.20} & 14.04 & 0.0 & 0.0 & 0.0 & \textbf{89.39} & 2.2 & 0.0 & 0.0 & 0.0 \\
    \textit{Standard} & 97.61 & 94.71 & 91.21 & 90.62 & 91.03 & 66.31 & 48.65 & 44.39 & 40.02 & 36.33 \\
    \textit{Curriculum} & 98.62 & \textbf{95.51} & 91.24 & 90.65 & 91.12 & 72.40 & 50.47 & 45.54 & 40.12 & 35.77\\
    \textbf{\textit{Dynamic}} & 97.96 & 95.34 & \textbf{91.63} & \textbf{91.27} & \textbf{91.47} & 72.17 & \textbf{52.81} & \textbf{48.06} & \textbf{42.40} & \textbf{37.26} \\
    \hline
    \vspace{-0.3 in}
    \end{tabular*}}
\end{table*}

\begin{table*}[!t]
    \caption{Black-box robustness (accuracy (\%) on black-box test attacks) of different defense models on MNIST and CIFAR-10 datasets.}
    \vspace{0.1 in}
    \centering
    \label{tab:black_box_accuracy}
    \scalebox{0.99}[0.99]{
    \begin{tabular*}{\textwidth}{ @{\extracolsep{\fill}} lcccc|cccc}
    \hline
    \multirow{2}[4]{*}{Defense}
    &\multicolumn{4}{c}{MNIST} &\multicolumn{4}{c}{CIFAR-10} \\
    \cmidrule(lr){2-5} \cmidrule(lr){6-9}
    & FGSM & PGD-10 & PGD-20 & C\&W$_{\infty}$ & FGSM & PGD-10 & PGD-20 & C\&W$_{\infty}$  \\
    \hline
    \textit{Standard} & 96.12 & 95.73 & 95.73 & 97.20 & 65.65 & 65.80 & 65.60 & 66.12\\
    \textit{Curriculum} & 96.59 & 95.87 & 96.09 & 97.52 & 71.25 & 71.44 & 71.13 & 71.94\\
    \textbf{\textit{Dynamic}} & \textbf{97.60} & \textbf{97.01} & \textbf{96.97} & \textbf{98.36} & \textbf{71.95} & \textbf{72.15} & \textbf{72.02} & \textbf{72.85} \\
    \hline
    \vspace{-0.15 in}
    \end{tabular*}}
\end{table*}

\textbf{Defense Settings.} For MNIST, defense models use a 4-layer CNN: 2 convolutional layers followed by 2 dense layers. Batch normalization (BatchNorm) \cite{ioffe2015batch} and max-pooling (MaxPool) are applied after each convolutional layer. For CIFAR-10, defense models adopt an 8-layer CNN architecture: 3 convolutional blocks followed by 2 dense layers, with each convolutional block has 2 convolutional layers. BatchNorm is applied after each convolutional layer, and MaxPool is applied after every convolutional block. Defense models for both MNIST and CIFAR-10 are trained using Stochastic Gradient Descent (SGD) with momentum 0.9, weight decay $10^{-4}$ and an initial learning rate of 0.01. The learning rate is divided by 10 at the 20-th and 40-th epoch for MNIST (50 epochs in total), and at the 60-th and 100-th epoch for CIFAR-10 (120 epochs in total). All images are normalized into [0, 1].

Except the \textit{Unsecured} model, other defense models including our proposed \textit{Dynamic} model are all trained under the same PGD adversarial training scheme: 10-step PGD attack with random start (adding an initial random perturbation of $[-\epsilon, \epsilon]$ to the normal examples before the PGD perturbation) and step size $\epsilon /4$. 
The maximum perturbation is set to $\epsilon=0.3$ for MNIST, and $\epsilon = 8/255$ for CIFAR-10, which is a standard setting for adversarial defense \cite{athalye2018obfuscated, madry2017towards}. 
For \textit{Dynamic} model, we set $c_{max}=0.5$ for both MNIST and CIFAR-10, and $T'=40$ for MNIST and $T'=100$ for CIFAR-10.
Other parameters of the baselines are configured as per their original papers.

\textbf{White-box Robustness.}
For MNIST and CIFAR-10, the attacks used for white-box setting are generated from the original test set images by attacking the defense models using 4 attacking methods: FGSM, PGD-10 (10-step PGD), PGD-20 (20-step PGD), and C\&W$_{\infty}$ ($L_{\infty}$ version of C\&W optimized by PGD for 30 steps). In the white-box setting, all attacking methods have full access to the defense model parameters and are constrained by the same maximum perturbation $\epsilon$. We report the classification accuracy of a defense model under white-box attacks as its white-box robustness.

The white-box results are reported in Table \ref{tab:white_box_accuracy}. On both datasets, the \textit{Unsecured} model achieves the best test accuracy on clean (unperturbed) images. However, it is not robust to adversarial examples --- accuracy drops to zero on strong attacks like PGD-10/20 or C\&W$_\infty$. The proposed \textit{Dynamic} model almost achieves the best robustness among all the defense models. Comparing the robustness on MNIST and CIFAR-10, the improvements are more significant on CIFAR-10. 
This may because MNIST consisting of only black-white digits is a relatively simple dataset where different defense models all work comparably well.
Compared to \textit{Standard} adversarial training, \textit{Dynamic} training with convergence quality controlled adversarial examples improves the robustness to a certain extent, especially on the more challenging CIFAR-10 with natural images. This robustness gain is possibly limited by the capacity of the small model (only an 8-layer CNN).  Thus we shortly show a series of experiments on WideResNet \cite{zagoruyko2016wide} where the power of the \textit{Dynamic} strategy is fully unleashed. In Table \ref{tab:white_box_accuracy}, we see that \textit{Curriculum} improves robustness against weak attacks like FGSM but is less effective against strong attacks like PGD/C\&W$_\infty$.

\textbf{Benchmarking the State-of-the-art on WideResNet.}
To analyze the full power of our proposed \textit{Dynamic} training strategy and also benchmark the state-of-the-art robustness on CIFAR-10, we conduct experiments on a large capacity network WideResNet \cite{zagoruyko2016wide} (10 times wider than standard ResNet \cite{he2016deep}), using the same settings as \citet{madry2017towards}. The WideResNet achieves an accuracy of 95.2\% on clean test images of CIFAR-10. For comparison, we include \textit{Madry's} WideResNet adversarial training and the \textit{Curriculum} model. White-box robustness against FGSM, PGD-20 and C\&W$_{\infty}$ attacks is shown in Table \ref{tab:madry_setting}. Our proposed \textit{Dynamic} model demonstrates a significant boost over \textit{Madry's} WideResNet adversarial training on FGSM and PGD-20, while \textit{Curriculum} model only achieves slight gains respectively. For the strongest attack C\&W$_{\infty}$, \textit{Curriculum}'s robustness decreases by 4\% compared to \textit{Madry's}, while \textit{Dynamic} model achieves the highest robustness.

\begin{table}[!t]
    \vspace{-0.1 in}
    \centering
    \caption{White-box robustness (\%) of different defense models on CIFAR-10 dataset using WideResNet setting in Madry's baselines.}
    \label{tab:madry_setting}
    \vspace{0.1 in}
    \begin{tabular}{l|ccccc}
    \hline
        Defense & Clean & FGSM & PGD-20 & C\&W$_{\infty}$ \\\hline
        \textit{Madry's}  & \textbf{87.3} & 56.1 & 45.8 & 46.8 \\
        \textit{Curriculum} & 77.43 & 57.17 & 46.06 & 42.28 \\
        \textbf{\textit{Dynamic}} & 85.03 & \textbf{63.53} & \textbf{48.70} & \textbf{47.27} \\
    \hline
    \end{tabular}
    \vspace{-0.2 in}
\end{table}

\textbf{Black-box Robustness.} Black-box test attacks are generated on the original test set images by attacking a surrogate model with architecture that is either a copy of the defense model (for MNIST) or a more complex ResNet-50 \cite{he2016deep} model (for CIFAR-10). Both surrogate models are trained separately from the defense models on the original training sets using \textit{Standard} adversarial training (10-step PGD attack with a random start and step size $\epsilon /4$). The attacking methods used here are the same as the white-box evaluation: FGSM, PGD-10, PGD-20, and C\&W$_{\infty}$. 

The robustness of different defense models against black-box attacks is reported in Table \ref{tab:black_box_accuracy}. Again, the proposed \textit{Dynamic} achieves higher robustness than the other defense models. 
\textit{Curriculum} also demonstrates a clear improvement over \textit{Standard} adversarial training. The distinctive robustness boosts of \textit{Dynamic} and \textit{Curriculum} indicate that training with weak adversarial examples at the early stage can improve the final robustness. 

Compared with the white-box robustness in Table \ref{tab:white_box_accuracy}, all defense models achieve higher robustness against black-box attacks, even the CIFAR-10 black-box attacks which are generated based on a much more complex ResNet-50 network (the defense network is only an 8-layer CNN). This implies that black-box attacks are indeed less powerful than white-box attacks, at least for the tested attacks. It is also observed that robustness tends to increase from weak attacks like FGSM to stronger attacks like C\&W$_\infty$. This implies that stronger attacks tend to have less transferability, an observation which is consistent with  \citet{madry2017towards}.

\subsection{Further Analysis}\label{sec:understanding}
\textbf{Different Maximum Perturbation Constraints $\epsilon$.}
We analyze the robustness of defense models \textit{Standard}, \textit{Curriculum} and the proposed \textit{Dynamic}, under different maximum perturbation constraints $\epsilon$ on CIFAR-10. For efficiency, we use the same 8-layer CNN defense architecture as in Sec. \ref{sec:robustness}. We see in Figure \ref{fig:robustness_epsilon} the white-box robustness of defense models trained with $\epsilon = 8/255$ against different PGD-10 attacks with varying $\epsilon \in [2/255, 8/255]$.  \textit{Curriculum} and \textit{Dynamic} models substantially improve the robustness of \textit{Standard} adversarial training, a result consistent with Sec. \ref{sec:robustness}.  \textit{Dynamic} training is better against stronger attacks with larger perturbations ($\epsilon = 8/255$) than \textit{Curriculum}. \textit{Curriculum} is effective on attacks with smaller perturbations ($\epsilon = 4/255, 2/255$), as similar performance to \textit{Dynamic}. We also train the defense models with different $\epsilon \in [2/255, 8/255]$, and then test their white-box robustness under the same $\epsilon$ (all defense models will tend to have similar low robustness if testing $\epsilon$ is larger than training $\epsilon$). As illustrated in Figure \ref{fig:robustness_setting}, training with weak attacks at the early stages might have a limit: the robustness gain tends to decrease when the maximum perturbation decreases to $\epsilon = 2/255$. This is not surprising given the fact that the inner maximization problem of adversarial training becomes more concave and easier to solve given the smaller $\epsilon$-ball. However, robustness for this extremely small scale perturbation is arguably less interesting for secure deep learning.

\begin{figure}[!htbp]
\centering
\begin{subfigure}{.48\linewidth}
  \centering  
  \includegraphics[width=\textwidth]{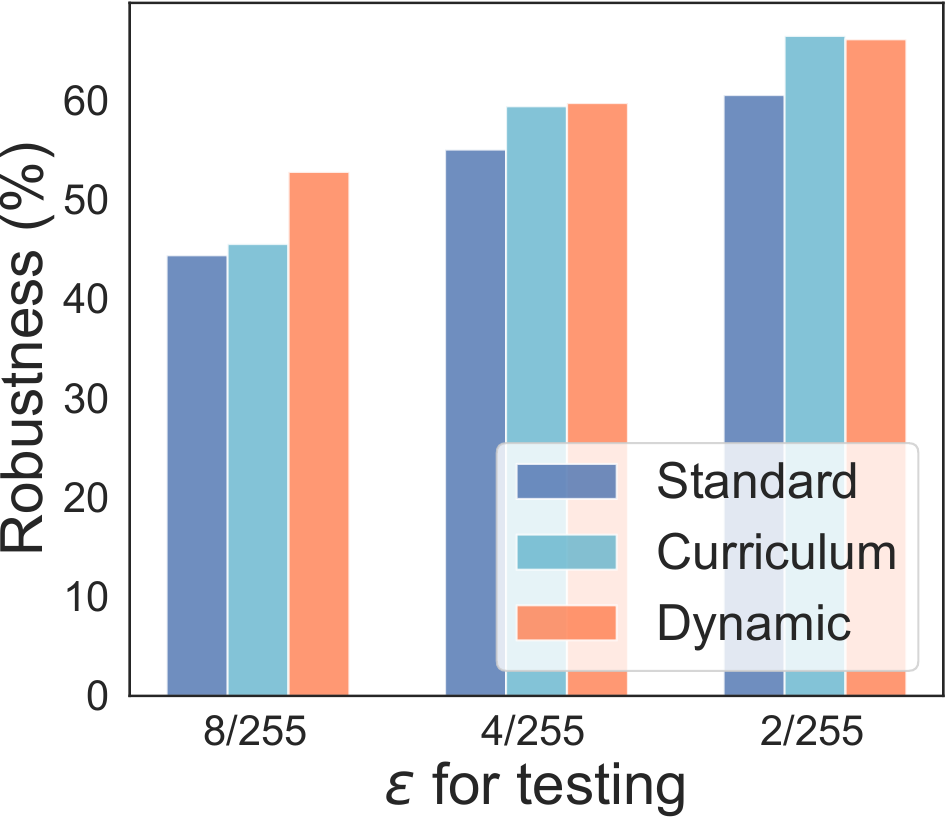}
  \vspace{-0.2 in}
  \caption{Robustness vs. testing $\epsilon$} 
  \label{fig:robustness_epsilon}
\end{subfigure}
\begin{subfigure}{.48\linewidth}
  \centering
  \includegraphics[width=\textwidth]{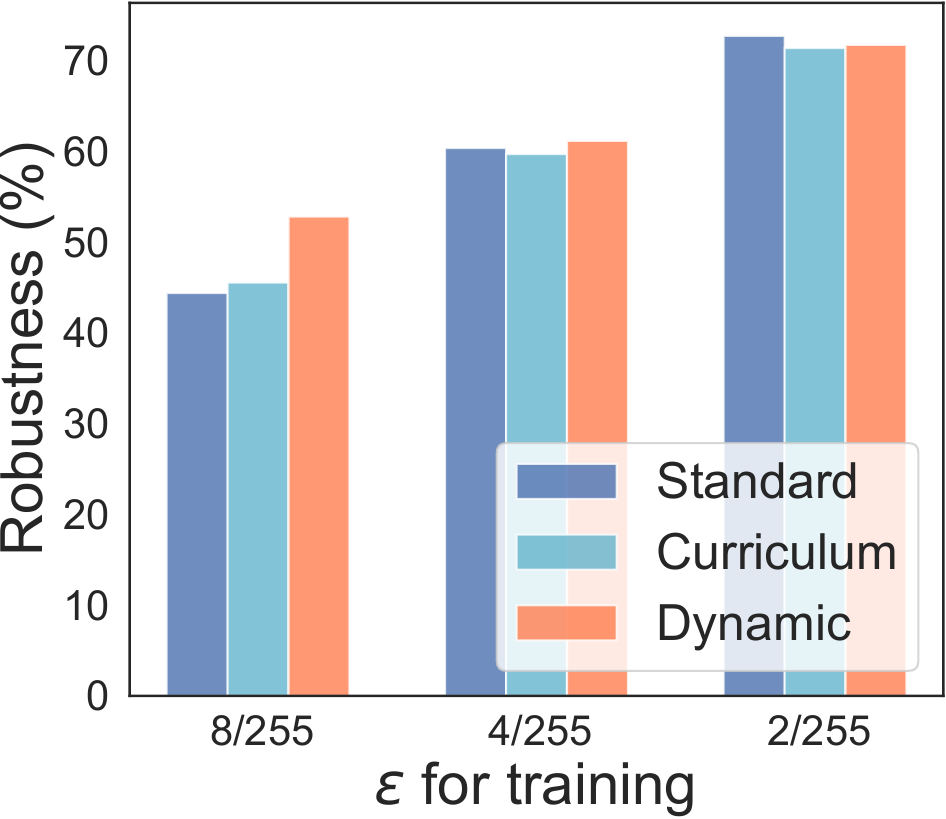}
  \vspace{-0.2 in}
  \caption{Robustness vs. training $\epsilon$} 
  \label{fig:robustness_setting}
\end{subfigure}
\vspace{-0.1 in}
\caption{(a): White-box robustness on PGD-10 attacks with different testing $\epsilon \in [2/255, 8/255]$; (b): White-box robustness of defense models trained on PGD-10 with different training $\epsilon \in [2/255, 8/255]$. 
}
\label{fig:epsilon}
\end{figure}

\textbf{Adversarial Training Process.}
To understand the learning dynamics of the 3 defense models, we plot the distribution of FOSC at different training epochs in Figures \ref{fig:analysis}. We choose epoch 10/60/100 for \textit{Standard} and \textit{Dynamic}, and epoch 60/90/120 for \textit{Curriculum} as it trains without adversarial examples at early epochs. We see both \textit{Curriculum} and \textit{Dynamic} learn with adversarial examples that are of increasing convergence quality (decreasing FOSC). 
The difference is that \textit{Dynamic} has more precise control over the convergence quality due to its use of the proposed criterion FOSC, demonstrating more concentrated FOSC distributions which are more separated at different stages of training. While for \textit{Curriculum}, the convergence quality of adversarial examples generated by the same number of perturbation steps can span a wide range of values (\textit{e.g.} the flat blue line in Figure \ref{fig:criterion_curriculum}), having both weak and strong adversarial examples.
Regarding the later stages of training (epoch 100/120), we see \textit{Dynamic} ends up with the best convergence quality (FOSC density over 40\% in Figure \ref{fig:criterion_dynamic}) followed by \textit{Curriculum} (FOSC density over 30\% in Figure \ref{fig:criterion_curriculum}) and \textit{Standard} (FOSC density less than 30\% in Figure \ref{fig:criterion_standard}), which is well aligned with their final robustness reported in Tables \ref{tab:white_box_accuracy} and \ref{tab:black_box_accuracy}.

\begin{figure}[!t]
\centering
\begin{subfigure}{.32\linewidth}
  \centering
  \includegraphics[width=\textwidth]{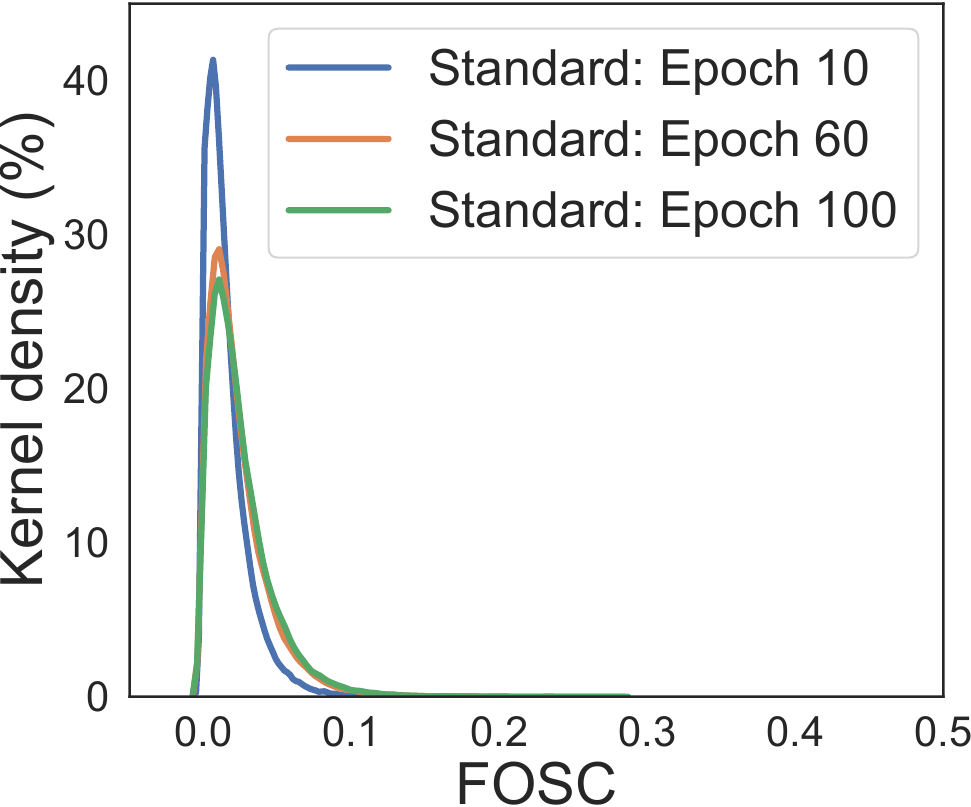}
  \vspace{-0.2 in}
  \caption{\textit{Standard}} 
  \label{fig:criterion_standard}
\end{subfigure}
\begin{subfigure}{.32\linewidth}
  \centering
  \includegraphics[width=\textwidth]{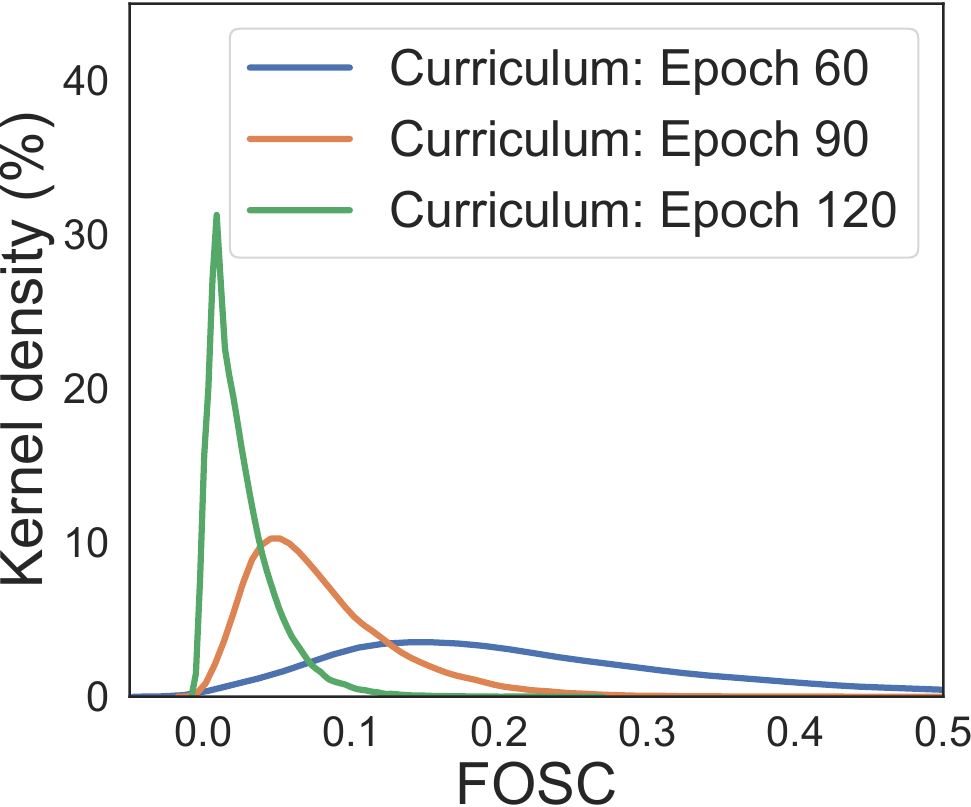}
  \vspace{-0.2 in}
  \caption{\textit{Curriculum}} 
  \label{fig:criterion_curriculum}
\end{subfigure}
\begin{subfigure}{.32\linewidth}
  \centering
  \includegraphics[width=\textwidth]{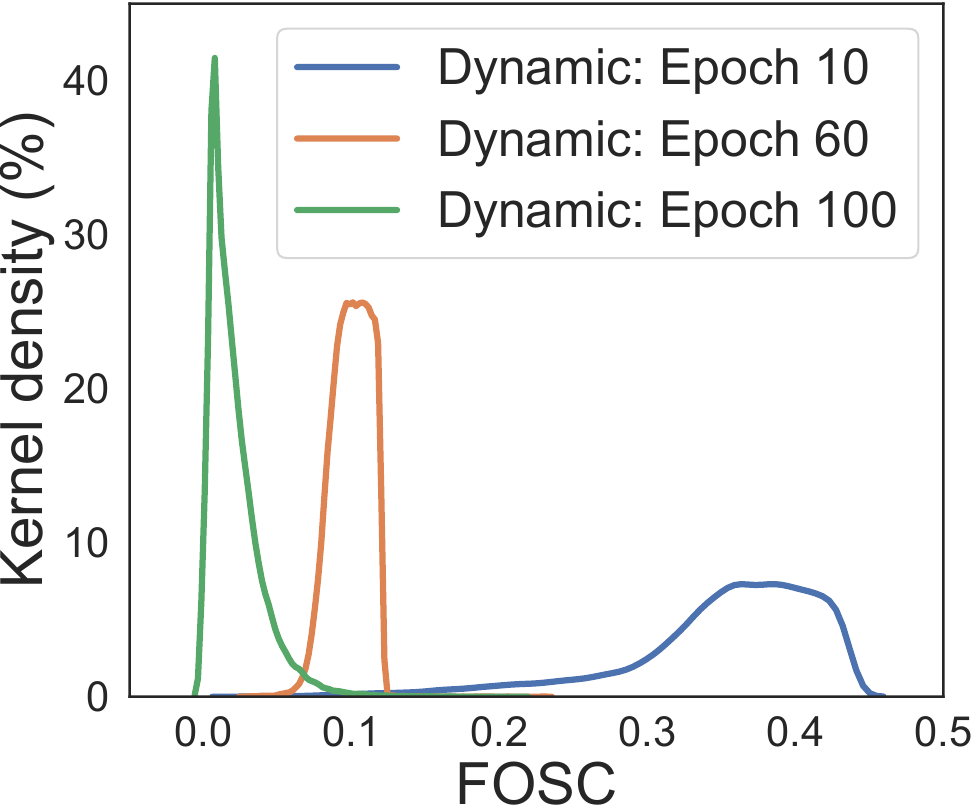}
  \vspace{-0.2 in}
  \caption{\textit{Dynamic}} 
  \label{fig:criterion_dynamic}
\end{subfigure}
\vspace{-0.1 in}
\caption{The distributions of FOSC at different epochs of training on CIFAR-10 with 10-step PGD of step size $\epsilon/4$ and $\epsilon=8/255$. 
}
\label{fig:analysis}
\vspace{-0.15 in}
\end{figure}

\section{Discussion and Conclusion}
In this paper, we proposed a criterion, First-Order Stationary Condition for constrained optimization (FOSC), to measure the convergence quality of adversarial examples found in the inner maximization of adversarial training. The proposed criterion FOSC is well correlated with adversarial strength and is more consistent than the loss. Based on FOSC, we found that higher robustness of adversarial training can be achieved by training on better convergence quality adversarial examples at the \textit{later stages}, rather than at the \textit{early stages}. Following that, we proposed a \textit{dynamic} training strategy and proved the convergence of the proposed approach for the overall min-max optimization problem under certain assumptions. On benchmark datasets, especially on CIFAR-10 under the WideResNet architecture for attacks with maximum perturbation constraint $\epsilon = 8/255$, our proposed dynamic strategy achieved a significant robustness gain against Madry's state-of-the-art baselines. 

Our findings imply that including very hard adversarial examples too early in training possibly inhibits DNN feature learning or encourages premature learning of overly complex features that provide less compression of patterns in the data. Experimental evidences also suggest that the later stages are more correlated with the final robustness, while the early stages are more associated with generalization. Therefore, we conjecture that higher robustness can be obtained by further increasing the diversity of weak adversarial examples in the early stages or generating more powerful adversarial examples in the later stages. The precise characterization of how the early and later stages interact with each other is still an open problem. We believe further exploration of this direction will lead to more robust models. 

\section*{Acknowledgements}
We would like to thank the anonymous reviewers for their helpful comments. We thank Jiaojiao Fan for pointing out a bug in the original proof of Theorem 1.

\bibliography{icml2019}
\bibliographystyle{icml2019}

\newpage
\onecolumn
\appendix

\section{Proof of Theorem \ref{thm:main}}

The proof of Theorem \ref{thm:main} is inspired by \citet{sinha2018certifying}. Before we prove this theorem, we need the following two technical lemmas.

\begin{lemma}\label{lemma:smoothness}
Under Assumptions \ref{assum:gradient Lipschitz} and \ref{assum:strongly concave}, we have $L_S(\btheta)$ is $L$-smooth where $L = L_{\theta x}L_{ x\theta}/\mu+L_{\theta \theta}$, i.e., for any $\btheta_1$ and $\btheta_2$, it holds
\begin{align*}
    L_S(\btheta_1) &\leq L_S(\btheta_2) + \la \nabla L_S(\btheta_2), \btheta_1 - \btheta_2 \ra + \frac{L}{2}\|\btheta_1-\btheta_2\|_2^2,\\
    \|\nabla L_S(\btheta_1)-\nabla L_S(\btheta_2)\|_2 &\leq L \|\btheta_1-\btheta_2\|_2
\end{align*}
\end{lemma}
\begin{proof}
By Assumption \ref{assum:strongly concave}, we have for any $\btheta_1, \btheta_2$, and $\xb_i^*(\btheta_1), \xb_i^*(\btheta_2)$, we have
\begin{align}\label{eq:Gu0001}
     f(\btheta_2,\xb_i^*(\btheta_1)) &\leq  f(\btheta_2,\xb_i^*(\btheta_2)) + \la \nabla_{\xb} f(\btheta_2,\xb_i^*(\btheta_2)), \xb_i^*(\btheta_1)-\xb_i^*(\btheta_2)\ra - \frac{\mu}{2} \|\xb_i^*(\btheta_1) - \xb_i^*(\btheta_2)\|_2^2\nonumber\\
    & \leq f(\btheta_2,\xb_i^*(\btheta_2))- \frac{\mu}{2} \|\xb_i^*(\btheta_1) - \xb_i^*(\btheta_2)\|_2^2,
\end{align}
where the inequality follows from $\la \nabla_{\xb} f(\btheta_2,\xb_i^*(\btheta_2)), \xb_i^*(\btheta_1)-\xb_i^*(\btheta_2)\ra \leq 0$.
In addition, we have
\begin{align}\label{eq:Gu0002}
    f(\btheta_2,\xb_i^*(\btheta_2)) \leq  f(\btheta_2,\xb_i^*(\btheta_1)) + \la \nabla_{\xb} f(\btheta_2,\xb_i^*(\btheta_1)), \xb_i^*(\btheta_2)-\xb_i^*(\btheta_1)\ra - \frac{\mu}{2} \|\xb_i^*(\btheta_1) - \xb_i^*(\btheta_2)\|_2^2
\end{align}
Combining \eqref{eq:Gu0001} and \eqref{eq:Gu0002}, we obtain
\begin{align}\label{eq:Gu0003}
    \mu \|\xb_i^*(\btheta_1) - \xb_i^*(\btheta_2)\|_2^2 &\leq \la \nabla_{\xb} f(\btheta_2,\xb_i^*(\btheta_1)),\xb_i^*(\btheta_2)-\xb_i^*(\btheta_1) \ra\nonumber\\
    &\leq \la \nabla_{\xb} f(\btheta_2,\xb_i^*(\btheta_1))-\nabla_{\xb} f(\btheta_1,\xb_i^*(\btheta_1)),\xb_i^*(\btheta_2)-\xb_i^*(\btheta_1) \ra\nonumber\\
    &\leq \| \nabla_{\xb} f(\btheta_2,\xb_i^*(\btheta_1))-\nabla_{\xb} f(\btheta_1,\xb_i^*(\btheta_1))\|_2 \|\xb_i^*(\btheta_2)-\xb_i^*(\btheta_1) \|_2\nonumber\\
    &\leq L_{x\theta}\|\btheta_2 - \btheta_1\|_2 \|\xb_i^*(\btheta_2)-\xb_i^*(\btheta_1) \|_2
\end{align}
where the second inequality holds because $\la \nabla_{\xb} f(\btheta_1,\xb_i^*(\btheta_1)), \xb_i^*(\btheta_2)-\xb_i^*(\btheta_1)\ra \leq 0$, the third inequality follows from Cauchy–Schwarz inequality, and the last inequality holds due to Assumption \ref{assum:gradient Lipschitz}.  
\eqref{eq:Gu0003} immediately yields
\begin{align}\label{eq:Gu0004}
    \|\xb_i^*(\btheta_1) - \xb_i^*(\btheta_2)\|_2 \leq \frac{L_{x\theta}}{\mu}\|\btheta_2 - \btheta_1\|_2.
\end{align}
Then we have for $i \in [n]$,
\begin{align}\label{eq:Gu0004}
    \|\nabla_\btheta f(\btheta_1,\xb_i^*(\btheta_1)) - \nabla_\btheta f(\btheta_2,\xb_i^*(\btheta_2))\|_2 &\leq \|\nabla_\btheta f(\btheta_1,\xb_i^*(\btheta_1))-\nabla_\btheta f(\btheta_1,\xb_i^*(\btheta_2))\|_2 \nonumber\\
    &\qquad+ \|\nabla_\btheta f(\btheta_1,\xb_i^*(\btheta_2))-\nabla_\btheta f(\btheta_2,\xb_i^*(\btheta_2))\|_2\nonumber\\
    &\leq L_{\theta x} \|\xb_i^*(\btheta_1)-\xb_i^*(\btheta_2)\|_2 + L_{\theta \theta} \|\btheta_1-\btheta_2\|_2\nonumber\\
    &= \bigg(\frac{L_{\theta x}L_{ x\theta}}{\mu}+L_{\theta \theta}\bigg)\|\btheta_1-\btheta_2\|_2
\end{align}
where the first inequality follows from triangle inequality, the second inequality holds due to Assumption \ref{assum:gradient Lipschitz}, and the last inequality is due to \eqref{eq:Gu0004}. 
Finally, by the definition of $L_S(\btheta)$, we have
\begin{align*}
\|\nabla L_S(\btheta_1) - \nabla L_S(\btheta_2)\|_2
    & \leq \bigg\|\frac{1}{n}\sum_{i=1}^n\nabla_\btheta f(\btheta_1,\xb_i^*(\btheta_1)) - \frac{1}{n}\sum_{i=1}^n\nabla_\btheta f(\btheta_2,\xb_i^*(\btheta_2))\bigg\|_2\\
    & \leq \frac{1}{n}\sum_{i=1}^n \|\nabla_\btheta f(\btheta_1,\xb_i^*(\btheta_1))-\nabla_\btheta f(\btheta_2,\xb_i^*(\btheta_2))\|_2\\
    &\leq \bigg(\frac{L_{\theta x}L_{ x\theta}}{\mu}+L_{\theta \theta}\bigg)\|\btheta_1-\btheta_2\|_2,
\end{align*}
where the last inequality follows from \eqref{eq:Gu0004}.
This completes the proof.
\end{proof}



\begin{lemma}\label{lemma:error}
Under Assumptions \ref{assum:gradient Lipschitz} and \ref{assum:strongly concave}, the approximate stochastic gradient $\hgb(\theta)$ satisfies
\begin{align}
    \|\hgb(\btheta) - \gb(\btheta)\|_2 \leq L_{\theta x} \sqrt{\frac{\delta}{\mu}}.
\end{align}
\end{lemma}
\begin{proof}
We have
\begin{align}\label{eq:Gu0010}
    \|\hgb(\btheta) - \gb(\btheta)\|_2 &= \bigg\|\frac{1}{|\cB|}\sum_{i\in \cB} (\nabla_{\btheta} f(\btheta,\hxb_i(\btheta)) - \nabla \bbf_i(\btheta))\bigg\|_2\nonumber\\
    &\leq \frac{1}{|\cB|}\sum_{i\in \cB} \big\| \nabla_{\btheta} f(\btheta,\hxb_i(\btheta)) - \nabla_{\btheta} f(\btheta,\xb_i^*(\btheta))\big\|_2\nonumber\\
    &\leq \frac{1}{|\cB|}\sum_{i\in \cB} L_{\theta x}\|\hxb_i(\btheta) - \xb_i^*(\btheta)\|_2,
\end{align}
where the first inequality follows from triangle inequality, and the second inequality holds due to Assumption \ref{assum:gradient Lipschitz}.
By Assumption \ref{assum:strongly concave}, we have for any $\btheta$, and $\xb_i^*(\btheta), \hxb_i(\btheta)$, we have
\begin{align}\label{eq:Gu0005}
    \mu \|\xb_i^*(\btheta) - \hxb_i(\btheta)\|_2^2 \leq  \la \nabla_{\xb} f(\btheta,\xb_i^*(\btheta)) - \nabla_{\xb} f(\btheta,\hxb_i(\btheta)), \hxb_i(\btheta)-\xb_i^*(\btheta)\ra. 
\end{align}
Since $\hxb_i(\btheta)$ is a $\delta$-approximate maximizer of $f(\btheta,\hxb_i(\btheta))$, we have
\begin{align}\label{eq:Gu0006}
    \la \xb_i^*(\btheta) - \hxb_i(\btheta), \nabla_{\btheta} f(\btheta,\hxb_i(\btheta))\ra \leq \delta.
\end{align}
In addition, we have 
\begin{align}\label{eq:Gu0007}
    \la \hxb_i(\btheta)-\xb_i^*(\btheta), \nabla_{\xb} f(\btheta,\xb_i^*(\btheta)) \ra \leq 0.
\end{align}
Combining \eqref{eq:Gu0006} and \eqref{eq:Gu0007} gives rise to
\begin{align}\label{eq:Gu0008}
    \la \hxb_i(\btheta)-\xb_i^*(\btheta), \nabla_{\xb} f(\btheta,\xb_i^*(\btheta)) - \nabla_{\btheta} f(\btheta,\hxb_i(\btheta)) \ra \leq \delta.
\end{align}
Substitute \eqref{eq:Gu0008} into \eqref{eq:Gu0005}, we obtain
\begin{align*}
    \mu \|\xb_i^*(\btheta) - \hxb_i(\btheta)\|_2^2  \leq \delta,
\end{align*}
which immediately yields
\begin{align}\label{eq:Gu0009}
    \|\xb_i^*(\btheta) - \hxb_i(\btheta)\|_2 \leq \sqrt{\frac{\delta}{\mu}}.
\end{align}
Substitute \eqref{eq:Gu0009} into \eqref{eq:Gu0010}, we obtain
\begin{align*}
    \|\hgb(\btheta) - \gb(\btheta)\|_2 \leq L_{\theta x} \sqrt{\frac{\delta}{\mu}},
\end{align*}
which completes the proof.
\end{proof}

Now we are ready to prove Theorem \ref{thm:main}. 
\begin{proof}[Proof of Theorem \ref{thm:main}]
Let $\bar f(\btheta) = 1/n\sum_{i=1}^n\min_{\xb_i} f(\btheta,\xb_i) = 1/n\sum_{i=1}^n f(\btheta,\xb_i^*)$. By Lemma \ref{lemma:smoothness}, we have
\begin{align*}
    L_S(\btheta^{t+1}) &\leq L_S(\btheta^t) + \la \nabla L_S(\btheta^{t}), \btheta^{t+1} - \btheta^t \ra + \frac{L}{2} \|\btheta^{t+1}-\btheta^t\|_2^2 \\
    &= L_S(\btheta^t) -\eta_t \|\nabla L_S(\btheta^{t})\|_2^2  + \frac{L\eta_t^2}{2} \|\hgb(\btheta^t)\|_2^2 +\eta_t \la \nabla L_S(\btheta^{t}), \nabla L_S(\btheta^{t}) - \hgb(\btheta^{t}) \ra\\
     & = L_S(\btheta^t) -\eta_t\bigg(1-\frac{L\eta_t}{2}\bigg) \|\nabla L_S(\btheta^{t})\|_2^2  + \eta_t\bigg(1-L\eta_t\bigg)\la \nabla L_S(\btheta^{t}), \nabla L_S(\btheta^{t}) - \hgb(\btheta^{t}) \ra\\
     &\qquad + \frac{L\eta_t^2}{2} \|\hgb(\btheta^t) - \nabla L_S(\btheta^t)\|_2^2\\
     &= L_S(\btheta^t) -\eta_t\bigg(1-\frac{L\eta_t}{2}\bigg) \|\nabla L_S(\btheta^{t})\|_2^2  + \eta_t\bigg(1-L\eta_t\bigg)\la \nabla L_S(\btheta^{t}), \gb(\btheta^t) - \hgb(\btheta^{t}) \ra\\ 
     &\qquad + \eta_t\bigg(1-L\eta_t\bigg)\la \nabla L_S(\btheta^{t}), \nabla L_S(\btheta^{t}) -\gb(\btheta^t)  \ra
      + \frac{L\eta_t^2}{2} \|\hgb(\btheta^t) -\gb(\btheta^t)+ \gb(\btheta^t)- \nabla L_S(\btheta^t)\|_2^2\\
      &\leq L_S(\btheta^t) -\frac{\eta_t}{2} \|\nabla L_S(\btheta^{t})\|_2^2  + \frac{\eta_t}{2}\bigg(1-L\eta_t\bigg)\|\hgb(\btheta^t) -\gb(\btheta^t)\|_2^2 \\
     &\qquad + \eta_t\bigg(1-L\eta_t\bigg)\la \nabla L_S(\btheta^{t}), \nabla L_S(\btheta^{t}) -\gb(\btheta^t)  \ra
      + L\eta_t^2\big( \|\hgb(\btheta^t) -\gb(\btheta^t)\|_2^2 + \|\gb(\btheta^t)- \nabla L_S(\btheta^t)\|_2^2\big),
\end{align*}
where the last inequality is due to the Young's inequality. Note that we have $\eta_t \leq 1/L$ because we choose $\eta_t = \eta =  \sqrt{\Delta/(TL\sigma^2)}$ and $T \geq (\Delta L)/\sigma^2$.
Taking expectation on both sides of the above inequality conditioned on $\btheta^t$, we have
\begin{align}\label{eq:induction}
    \EE [L_S(\btheta^{t+1}) - L_S(\btheta^t) | \btheta^t] &\leq -\frac{\eta_t}{2} \|\nabla L_S(\btheta^{t})\|_2^2 + \frac{\eta_t}{2}\bigg(1+L \eta_t\bigg)  \frac{L_{\theta x}^2\delta}{\mu} + L\eta_t^2 \sigma^2\notag\\
    &\leq -\frac{\eta_t}{2} \|\nabla L_S(\btheta^{t})\|_2^2 + \eta_t \frac{L_{\theta x}^2\delta}{\mu} + L\eta_t^2 \sigma^2,
\end{align}
where the first inequality uses the fact that $\EE[\gb(\btheta^t)] = \nabla L_S(\btheta^t)$, Assumption \ref{assum:variance}, and Lemma \ref{lemma:error}, and the second inequality uses the fact that $\eta_t \leq 1/L$.
Taking telescope sum of \eqref{eq:induction} over $t=0,\ldots, T-1$, we obtain that
\begin{align*}
     \sum_{t=0}^{T-1} \frac{\eta_t}{2} \EE\big[\|\nabla L_S(\btheta^{t})\|_2^2\big]\leq \EE[ L_S(\btheta^0)-L_S(\btheta^T)]  + \sum_{t=0}^{T-1} \eta_t \frac{L_{\theta x}^2\delta}{\mu} + L\sum_{t=0}^{T-1}\eta_t^2 \sigma^2.
\end{align*}
Recall that $\eta_t = \eta =  \sqrt{\Delta/(TL\sigma^2)}$ where $L = L_{\theta x}L_{ x\theta}/\mu+L_{\theta \theta}$, we can show that 
\begin{align*}
    \frac{1}{T}\sum_{t=0}^{T-1}\EE\big[\|\nabla L_S(\btheta^{t})\|_2^2\big] \leq 4\sigma\sqrt{\frac{ L\Delta}{T}} + \frac{ 2L_{\theta x}^2\delta}{\mu}.
\end{align*}

This completes the proof.
\end{proof}

\end{document}